%% file: main.tex
\documentclass[11pt]{article}

\usepackage{natbib}
\usepackage{fullpage}
\usepackage[utf8]{inputenc} 
\usepackage[T1]{fontenc}    
\usepackage{hyperref}       
\usepackage{url}            
\usepackage{booktabs}       
\usepackage{amsfonts}       

\usepackage{nicefrac}       
\usepackage{microtype}      
\usepackage{comment}
\usepackage{graphicx}
\usepackage{subfigure}
\usepackage{booktabs} 
\usepackage{mathtools}
\usepackage{xcolor}
\usepackage{algorithm}
\usepackage{tikz}
\usepackage{amssymb}
\usepackage{amsmath}
\usepackage{amsthm}
\usepackage{algorithmic}
\usepackage{wrapfig}

\renewcommand{\algorithmiccomment}[1]{\bgroup\hfill//~#1\egroup}

\newcommand{\compilehidecomments}{false}
\ifthenelse{ \equal{\compilehidecomments}{true} }{%
	\newcommand{\huazheng}[1]{}
}{
\newcommand{\huazheng}[1]{{\color{blue!50!black}  [\text{Huazheng:} #1]}}
}

\newcommand{\maxi}[1]{\mbox{maximize} & {#1 } & \\}

\newcommand{\st}{\mbox{subject to} }

\newcommand{\con}[1]{&#1 & \\}
\newcommand{\qcon}[2]{&#1, & \mbox{for } #2.  \\}
\newenvironment{lp}{\begin{equation}  \begin{array}{lll}}{\end{array}\end{equation} }
\newenvironment{lp*}{\begin{equation*}  \begin{array}{lll}}{\end{array}\end{equation*}}

\newcommand{\mt}{\mathsf{T}}
\DeclareMathOperator*{\argmax}{arg\,max}

\def \btheta {\bm \theta}

\def \btheta {\mathrm{\boldsymbol{\theta}}}

\def \bA {\mathbf{A}}

\def \cA {\mathcal{A}}

\def \bI {\mathbf{I}}

\newtheorem{example}{Example}
\newtheorem{definition}{Definition}
\newtheorem{remark}{Remark}
\newtheorem{theorem}{Theorem}
\newtheorem{corollary}{Corollary}
\newtheorem{lemma}{Lemma}
\newtheorem{claim}{Claim}

\title{
When Are Linear Stochastic Bandits Attackable?}

\author{%
  Huazheng Wang \thanks{Most work was done while with University of Virginia.}\\
  Princeton University\\
  \texttt{huazhengwang@gmail.com} \\
  \and
  Haifeng Xu\\
  University of Virginia\\
  \texttt{hx4ad@virginia.edu} \\
  \and
  Hongning Wang\\
  University of Virginia\\
  \texttt{hw5x@virginia.edu} \\  
}
\date{}
\begin{document}

\maketitle

\begin{abstract}
We study adversarial attacks on linear stochastic bandits: by manipulating the rewards, an adversary aims to control the behaviour of the bandit algorithm. Perhaps surprisingly, we first show that some attack goals can never be achieved. This is in a sharp contrast to  context-free stochastic bandits, and is intrinsically due to the correlation among arms in linear stochastic bandits. Motivated by this finding, this paper studies the \emph{attackability} of a $k$-armed linear bandit environment. We first provide a complete necessity and sufficiency characterization of attackability based on the geometry of the arms' context vectors. 
We then propose a two-stage attack method against LinUCB and Robust Phase Elimination. The method first asserts whether the given environment is attackable; and if yes, it poisons the rewards to force the algorithm to pull a target arm linear times using only a sublinear cost. Numerical experiments further validate the effectiveness and cost-efficiency of the proposed attack method. \end{abstract}

\input{./intro}

\input{./prelim}

\input{./oracle}
\input{./real}

\input{./exp}

\input{./related}

\section{Conclusion}
In this paper, we studied the problem of poisoning attacks in $k$-armed linear stochastic bandits.
Different from context-free stochastic bandits and $k$-armed linear contextual bandits where the environment is always attackable, we showed that some linear stochastic bandit environments are \emph{not} attackable due to the correlation among arms. We characterized the attackability condition as the feasibility of a CQP based on the geometry of the arms' context vectors. Our key insight is that given the ground-truth parameter $\btheta^*$, the adversary should perform oracle attack that lowers the reward of non-target arms in the null space of the target arm's convex vector $\tilde x$. 
Based on this insight, we proposed a two-stage null space attack without the knowledge of $\btheta^*$ and applied it against LinUCB and Robust Phase Elimination. We showed  that the proposed attack methods are effective and cost-efficient, both theoretically and empirically. 

As our future work, it is interesting to study the lower bound of attack cost in linear stochastic bandits and also design cost-optimal attack method with a matching upper bound. One idea is to design a multi-stage attack method following the doubling trick idea, which was brief discussed in Appendix \ref{sec:multi-stage}. Although the analysis could be much more challenging than our two-stage attack, it may lead to a lower attack cost as well as handling unknown time horizon $T$. 
Another intriguing direction is  to study algorithm-oblivious choice of the length of the first stage $T_1$ --- or more generally, algorithm-oblivious attack strategies --- that can achieve  sublinear cost for \emph{arbitrary} no-regret algorithm without the knowledge of $\btheta^*$.

\section*{Acknowledgements}

This work was supported in part by National Science Foundation Grant IIS-2128019, IIS-1618948, IIS-1553568,  Bloomberg Data Science Ph.D. Fellowship, and a Google Faculty Research Award.

\bibliographystyle{apalike}
\bibliography{main.bib}

\input{./supp}

\end{document}

%% file: intro.tex
\section{Introduction}



In a contextual bandit problem, a learner takes sequential actions to interact with an environment to maximize its received cumulative reward. As a natural and important variant, linear stochastic bandits \citep{Auer02,li2010contextual,Improved_Algorithm} assume the expected reward of an arm $a$ is a linear function of its feature vector $x_a$ and an unknown bandit parameter $\btheta^*$. A linear bandit algorithm thus adaptively improves its estimation of $\btheta^*$ based on the reward feedback on its pulled arms. Thanks to their sound theoretical guarantees and promising empirical performance, linear stochastic bandits have become a reference solution to many real-world problems, such as content recommendation and online advertisement  \citep{li2010contextual,chapelle2011empirical,durand2018contextual}.

Since bandit algorithms adapt their behavior according to its received feedback, such algorithms are susceptible to adversarial attacks, especially poisoning attacks. Under such an attack, a malicious adversary observes the pulled arm and its reward feedback, and then modifies the reward to misguide the bandit algorithm to pull a target arm, which is of the adversary's interest. Due to the wide applicability of bandit algorithms in practice as mentioned before, 
understanding the robustness of such algorithms under poisoning attacks becomes increasingly important \citep{jun2018adversarial,liu2019data,garcelon2020adversarial}. 

Most existing studies on adversarial attacks in bandits focused on the context-free stochastic multi-armed bandit (MAB) settings. \citet{jun2018adversarial} and \citet{liu2019data} showed that an adversary can force any MAB algorithm to pull a target arm linear times only using a logarithmic cost.
\citet{garcelon2020adversarial} showed the vulnerability of $k$-armed linear contextual bandits under poisoning attacks. 
Linear stochastic bandits are related to context-free stochastic bandits and linear contextual bandits. Interestingly, however,  there is no known result about attacks on linear stochastic bandit until now. This paper shall provide a formal explanation for this gap ---  the analysis of  attacks to linear stochastic bandits turns out to be significantly more challenging due to the correlation among arms; in fact,  \emph{some learning environment is provably unattackable}.

Specifically, we fill the aforementioned gap by studying  poisoning attacks on $k$-armed linear stochastic bandits, where an adversary modifies the reward using a sublinear attack cost to misguide the bandit algorithm to pull a target arm $\tilde x$ linear times. We first show that a linear stochastic bandit environment is \emph{not always efficiently attackable}\footnote{Throughout this paper, ``efficient attack'' means fooling the bandit algorithm to pull the target arm for linear times with a sublinear attack cost.  We will use \emph{attackable} and \emph{efficiently attackable} interchangeably, as the adversary normally only has a limited budget and needs to design a cost-efficient strategy.}, and its attackability is governed by the feasibility of finding a parameter vector $\tilde \btheta$, under which the rewards of all non-target arms are smaller than the reward of target arm $\tilde x$ and the reward of $\tilde x$ remains the same as that in the original environment specified by $\btheta^*$. 
Intuitively, to promote the target arm $\tilde x$, an adversary needs to lower the rewards of non-target arms in the \emph{null space} of $\tilde x$ by  $\tilde \btheta$, which might not be always feasible. We prove the feasibility of the resulting convex quadratic program is both \emph{sufficient} and \emph{necessary} for attacking a linear stochastic bandit environment. 

Inspired by our attackability analysis, we propose a two-stage attack framework against linear stochastic bandit algorithms and demonstrate its application against LinUCB~\citep{li2010contextual} and Robust Phase Elimination~\citep{bogunovic2021stochastic}: the former is one of the most widely used linear contextual bandit algorithms, and the latter is a robust version designed for settings with adversarial corruptions. 
In the first stage, our method collects a carefully calibrated amount of rewards on the target arm to assess whether the given environment is attackable. The decision is based on an ``empirical'' version of our feasibility characterization. If attackable, i.e., there exists a feasible solution $\tilde \btheta$, in the second stage the method depresses the rewards the bandit algorithm receives on non-target arms based on $\tilde \btheta$, to fool the bandit algorithm to recognize the target arm as optimal. 
We prove that in an attackable environment, both algorithms can be successfully manipulated with only a sublinear cost.

Our main contributions can be summarized as follows: 
\begin{itemize}
    \item We characterize the sufficient and necessary conditions about when a stochastic linear bandit environment is attackable as the feasibility of a convex quadratic program. En route to proving the sufficiency, we also provide an oracle attack method that can attack \emph{any} no-regret learning algorithm given the knowledge of ground-truth bandit parameter $\btheta^*$.
    If the environment is unattackable, i.e., the program is infeasible, our necessity proof implies that even the vanilla LinUCB algorithm cannot be efficiently attacked. 
    A direct corollary of our characterization is that context-free stochastic MAB is always attackable, resonating the findings in \citep{jun2018adversarial,liu2019data}. 
    \item We propose a two-stage attack method that works \emph{without} the knowledge of ground-truth bandit parameter. In the first stage, the algorithm detects the attackability of the environment and estimates the model parameter. In the second stage, it solves for a working solution $\tilde\btheta$ and attacks accordingly. Our theoretical analysis shows this attack method is effective against LinUCB~\citep{li2010contextual} and Robust Phase Elimination~\citep{bogunovic2021stochastic}, i.e., pulling the target arm $T-o(T)$ times using $o(T)$ cost when the environment is attackable. 
\end{itemize}

%% file: prelim.tex
\section{Preliminaries}

\textbf{Linear stochastic bandit.}
We study poisoning attacks to the fundamental $k$-armed linear stochastic bandit problem \citep{Auer02}, where a bandit algorithm sequentially interacts with an environment for $T$ rounds. In each round $t$, the algorithm pulls an arm $a_t \in [k] = \{1,\cdots,k\}$ from a set $\mathcal{A} = \{x_i\}^k_{i=1}$ with $k$ arms, and receives reward $r_{a_t}$ from the environment. 
Each arm $a$ is associated with a  $d$-dimensional context vector $x_a \in \mathbb{R}^{d}$; and the observed reward follows a linear mapping $r_{a_t} = x_{a_t}^\mt \btheta^* + \eta_t$, where $\btheta^* \in \mathbb{R}^{d}$ is a common unknown bandit parameter vector and $\eta_t$ is an $R$-sub-Gaussian noise term. 
We assume context vectors and parameters are all bounded; and for convenience and without loss of generality, we assume $||x_i||_2 \leq 1$ and $\Vert \btheta^*\Vert_2 \leq 1$. 
The performance of a bandit algorithm is evaluated by its pseudo-regret, which is defined as $R_T(\btheta^*) = \sum_{t=1}^T(x^{\mt}_{a^*}\btheta^* - x^{\mt}_{a_t}\btheta^*)$,
where $a^*$ is the best arm according to $\btheta^*$, i.e., ${a^*} = \arg \max_{a\in[k]} [ x_a^\mt \btheta^* ]$. 

\paragraph{LinUCB.} LinUCB \citep{li2010contextual, Improved_Algorithm} is a classical algorithm for linear stochastic bandit. It estimates a bandit model parameter $\hat\btheta$ using ridge regression, i.e., $\hat\btheta_t = \bA_t^{-1} \sum_{i=1}^t x_{a_i} r_i$, where $\bA_t = \sum_{i=1}^t x_{a_i} x_{a_i}^\mt + \lambda\bI$ and $\lambda$ is the L2-regularization coefficient. We use $\lVert x \rVert_{\bA} = \sqrt{x^\mt \bA x}$ to denote the matrix norm of vector $x$. The confidence bound about reward estimation on arm $x$ is defined as $\text{CB}_t(x) = \alpha_t \lVert x \rVert_{\bA_t^{-1}}$, where $\alpha_t$ is a high probability bound of $\lVert \btheta^* - \hat\btheta_t \rVert_{\bA_t}$. In each round $t$, LinUCB pulls an arm with the highest upper confidence bound, i.e., $a_t = \argmax_{a\in[k]} [x_a^\mt \hat\btheta_t + \text{CB}_t(x_a)]$ to balance the exploration-exploitation trade-off. LinUCB algorithm achieves a sublinear upper regret bound, i.e., $R_T = \tilde O(\sqrt{T})$ ignoring the logarithmic term. 

\paragraph{Threat model.} The goal of an adversary is to fool a linear stochastic bandit algorithm to pull the target arm  $\tilde x \in \mathcal{A}$ for $T-o(T)$ times. 
Like most recent works in this space~\citep{jun2018adversarial,liu2019data,garcelon2020adversarial,zhang2020adaptive}, we also consider the widely studied poisoning attack on the rewards: after arm $a_t$ is pulled by the bandit algorithm, the adversary modifies the realized reward $r_{a_t}$ from the environment by $\Delta r_t$ into $\tilde r_{a_t}$, i.e., $\tilde r_{a_t} = r_{a_t} + \Delta r_t$, and feeds the manipulated reward  $\tilde r_{a_t}$ to the algorithm. Naturally, the adversary aims to achieve its attack goal with minimum attack cost, defined as $C(T) =\sum_{t=1}^T \vert\Delta r_t\vert$. 
By convention, an attack strategy is said to be \emph{efficient}  if it uses only a sublinear total cost, i.e., $C(T) = o(T)$.  

We conclude the preliminaries with an important remark about a key difference between attacking linear stochastic bandit studied in this paper and attacking $k$-armed linear contextual bandit setting recently studied in~\citep{garcelon2020adversarial}. In linear contextual bandits, all arms share a context vector at each round but each arm has its own (to-be-estimated) bandit parameter. Therefore, the reward manipulation at a round $t$ will only affect the parameter estimation of the pulled arm $a_t$, but not any other arms'. This ``isolates'' the attack's impact in different arms. In contrast, in linear stochastic bandit, all arms share the same bandit parameter but have different context vectors. And thus manipulating the reward of any arm will alter the shared bandit parameter estimation, which will then affect the reward estimation of all arms due to the correlation among their context vectors. Such coupled effect of adversarial manipulation from the pulled arm $a_t$ to all other arms is unique in linear stochastic bandits, and makes its analysis of attack much more challenging. This is also the fundamental reason that some linear stochastic bandit environment may not be attackable (see our illustration in Example \ref{ex:attackability}).

%% file: oracle.tex
\section{The Attackability of A Linear Stochastic Bandit Environment} \label{sec:oracle}
We study the attackability of a linear stochastic bandit \emph{environment}. At the first glance, one might wonder whether \emph{attackability} is the property of a bandit \emph{algorithm} rather than a property of the environment, since if an algorithm can be attacked, we should have ``blamed'' the algorithm for not being robust. A key finding of this work is attackability is also a property of the learning environment; and in other words, \emph{not} all environments are attackable.

\begin{definition}[Attackability of a $k$-Armed Linear Stochastic Bandit Environment]\label{def:attackable}
A $k$-armed linear stochastic bandit environment $\langle\mathcal{A} = \{x_i\}^k_{i=1}, \btheta^*\rangle$ is  attackable with respect to the target arm $\tilde{x} \in \mathcal{A}$ if for \emph{any} no-regret learning algorithm, there exists an attack method that uses $o(T)$ attack cost and fools the algorithm to pull arm $\tilde{x}$ at least $T - o(T)$ times with high probability\footnote{Typically, the high probability refers to $1-\delta$ probability for an arbitrarily small $\delta$. Please see theorems later for rigorous statements.} for any $T$ large enough, i.e., $T$ larger than a constant $T_0$. 
\end{definition} 

We make a few remarks about the above definition of attackability. First, this definition is all about the bandit environment $\langle\mathcal{A}, \btheta^*\rangle$ and the target arm $\tilde{x}$, but independent of any specific bandit algorithm. Second, if an attack method can only fool a bandit algorithm to pull the target arm $\tilde{x}$ under a few different time horizons $T$, it does not count as a successful attack -- it has to succeed for infinitely many time horizons. Third, by reversing the order of quantifiers, we obtain the assertion that a bandit environment is not attackable w.r.t. the target arm $\tilde{x}$  if \emph{there exists some no-regret learning algorithm} such that no attack method can use $o(T)$ attack cost to fool the algorithm to pull arm $\tilde{x}$ at least $T - o(T)$ times  with high probability for any $T$ large enough. 

The following simple yet insightful example illustrates that there are indeed linear stochastic bandit environment setups where some no-regret learning algorithm \emph{cannot} be attacked.  

\begin{example} [An unattackable environment]\label{ex:attackability}
Figure~\ref{fig:example} shows a three-arm environment with $\mathcal{A}=\{x_1 = (0, 1), x_2 = ( 1, 2), x_3 = (-1, 2)\}$. Suppose the target arm $\tilde x = x_1$ and the ground-truth bandit parameter $\btheta^* = (1, 1)$\footnote{One can re-scale all vectors with norms smaller than 1, e.g., divide each dimension by 10, without changing the conclusion that the environment is unattackable. }. The expected true rewards of the arms are $r_1^* = 1, r_2^*=3, r_3^* = 1$ and $x_2$ is the best arm in this environment. 

Based on Definition \ref{def:attackable}, we will need to identify a no-regret learning algorithm that   cannot be attacked in this environment, and we argue that LinUCB is such an algorithm. Suppose, for the sake of contradiction, that there exists an efficient attack which fools LinUCB to pull $x_1$ $T - o(T)$ times. LinUCB then must estimate $\btheta^*$ in the $x_1$'s direction almost accurately as $T$ becomes large, since the $\Omega(T)$ amount of true reward feedback in $x_1$ direction will ultimately dominate the $o(T)$ adversarial manipulations. This suggests that the estimated parameter $\hat\btheta_t$ will be close to  $(\alpha, 1)$ for some $\alpha$. Since $(\alpha, 1)^\mt (x_2 + x_3) = 4$, implying that either $x_2$ or $x_3$ will have its estimated reward larger than $2$ (i.e., strictly larger than $x_1$'s estimated reward) for any $\alpha$. This shows that for large $T$, $x_1$ cannot be the  arm with the highest UCB during the execution of LinUCB, which causes a contradiction. Therefore,   this environment cannot be efficiently attacked with $o(T)$ cost. Here we give an intuitive argument about this environment with target arm $\tilde x$ is not attackable, while its formal proof is an instantiation of our Theorem \ref{theorem:characterization}.

\begin{figure}[t]
    \centering
     \includegraphics[width=5cm]{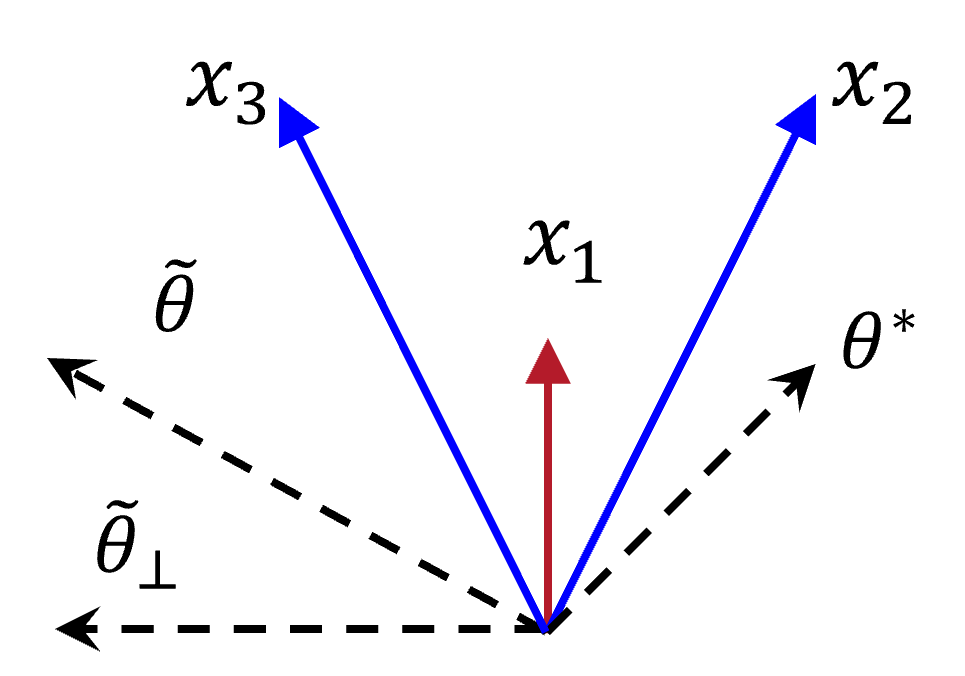}
    \caption{Illustration of attackability of a linear stochastic bandit environment.}
    \label{fig:example}
\end{figure} 

\end{example}

Note that when $\mathcal{A}=\{x_1, x_2\}$, the environment becomes attackable: as shown in   Figure~\ref{fig:example}, a feasible attack strategy is to perturb reward according to $\tilde\btheta =  (-2,1)$. The key idea is that in the null space of $x_1$, $\tilde\btheta_\perp$ reduces the reward of $x_2$ to make $x_1$ the best arm without changing the actual reward of $x_1$ from the environment.   The presence of arm $x_3$ prevents the existence of such a $\tilde\btheta_\perp$ (and therefore $\tilde\btheta$) and makes the environment unattackable.

The above example motivates us to study when a linear stochastic bandit environment is attackable. After all, only when facing an unattackable environment, we can ensure the existence of a no-regret algorithm  that will be robust to any $o(T)$ poisoning attacks. 

Next we show that there indeed exists a complete characterization about when a linear stochastic bandit environment is attackable. As Example \ref{ex:attackability} shows, the attackability of a bandit environment depends on the arm set $\mathcal{A} = \{x_i\}^k_{i=1}$, the target arm $\tilde x$, and the underlying bandit parameter $\btheta^*$. For clarity of presentation, in this section, we shall always assume that the adversary knows exactly the ground-truth bandit parameter $\btheta^*$ and thus the true expected reward of each arm. This is also called the \emph{oracle attack} in previous works~\citep{jun2018adversarial,rakhsha2020policy}. But in the next section, we will show that this assumption is not necessary: when the bandit environment is indeed attackable, we can design provably successful attacks even if the adversary does not know the underlying bandit parameter $\btheta^*$. 

We need the following convenient notation to state our results. Let \begin{equation}\label{eq:theta_parallel}
  \btheta^*_\parallel = \frac{\tilde x^\mt \btheta^*}{\lVert\tilde x\rVert^2_2}  \tilde x  
\end{equation} denote the projection of ground-truth bandit parameter $\btheta^*$ onto the targeted $\tilde x$ direction. Since the attackability depends on the target arm $\tilde{x}$ as well, we shall include the target arm $\tilde{x}$ as part of the bandit environment. The following theorem provides a clean characterization about the attackability of a linear stochastic bandit environment.  

\begin{theorem}[Characterization of Attackability]\label{theorem:characterization}
A bandit environment $\langle \mathcal{A}=\{x_i\}^k_{i=1}, \btheta^*, \tilde x \rangle$ is  attackable \emph{if and only if} the optimal objective  $\epsilon^*$ of the following convex quadratic program (CQP) 
satisfies $\epsilon^* > 0$. 
\begin{lp}\label{lp:attackability}
\maxi{\epsilon}
\st 
\qcon{   \tilde x^\mt    \btheta^*_\parallel   \geq \epsilon +   x_a^\mt  ( \btheta^*_\parallel +  \tilde\btheta_\perp)  }{ x_a \not = \tilde{x}}
\con{   \tilde x^\mt    \tilde\btheta_\perp  = 0 }
\con{   \Vert \btheta^*_\parallel +  \tilde\btheta_\perp\Vert_2 \leq 1}
\end{lp}
where  $\epsilon \in \mathbb{R} \text{ and }  \tilde\btheta_\perp \in \mathbb{R}^d$ are variables.  
\end{theorem}


Since the conceptual message of Theorem \ref{theorem:characterization} significantly differs from  previous studies on adversarial attacks in bandit algorithms, we would like to elaborate on its implications. 

First of all, we, for the first time, point out some learning environment is intrinsically robust. Even the vanilla LinUCB algorithm, as we will analyze in the proof of Theorem \ref{theorem:characterization}, cannot be efficiently attacked when CQP \eqref{lp:attackability} is not satisfied. Notably, although almost all previous works have focused on the vulnerability of bandit algorithms, e.g., by designing attacks against UCB, $\epsilon$-Greedy \citep{jun2018adversarial}, LinUCB \citep{garcelon2020adversarial}, it just so happens  that they were already studied under an attackable environment (see our Corollary \ref{cor-MAB}). To our best knowledge, the problem about the intrinsic robustness of a linear bandit environment has not been studied before and can be viewed as a complement to these previous works. Second, as we will show next, our proof techniques are also significantly different from existing ones, since what is central to our proof is to demonstrate that when CQP \eqref{lp:attackability} is not satisfied, there will exist a robust algorithm that cannot be efficiently attacked by \emph{any} adversary. This can be viewed as analyzing the robustness of certain bandit algorithms when $\epsilon^*\leq 0$ in CQP \eqref{lp:attackability}.  

Since CQP \eqref{lp:attackability} and its solutions  will show up very often in our later analysis, we provide the following definition for reference convenience. 

\begin{definition}[Attackability Index and Certificate]
The optimal objective $\epsilon^*$ of  CQP \eqref{lp:attackability} is called the \emph{attackability index} and the optimal solution $\tilde\btheta_\perp$ to CQP \eqref{lp:attackability} is called the \emph{attackability certificate.}\footnote{We sometimes omit ``attackability'' when it is clear from the context, and simply mention \emph{index} and \emph{certificate}.}
\end{definition} 
We should note both the index $\epsilon^*$ and certificate $\tilde\btheta_\perp$ are intrinsic to the bandit environment $\langle \mathcal{A}=\{x_i\}^k_{i=1}, \btheta^*, \tilde x \rangle$, and are irrelevant to any bandit algorithms used. As we will see  in the next section when designing  attack algorithms \emph{without} the knowledge of $\btheta^*$, the index $\epsilon^*$ will determine how difficult it is to attack the environment. 

\begin{algorithm}[h]
    \caption{Oracle Null Space Attack}\label{alg:oracleattack}
    \begin{algorithmic}[1]
    \STATE \textbf{Inputs:} $T, \btheta^*$
    \STATE \textbf{Initialize:}
    \IF[Attackability Test]{Optimal objective $\epsilon^*$ of CQP \eqref{lp:attackability} $>0$ } 
        \STATE Find the optimal solution $\tilde\btheta_\perp$
        \STATE Set $\tilde\btheta = \btheta^*_{\parallel} + \tilde\btheta_\perp$
        \ELSE
        \STATE \textbf{return} Not attackable
        \ENDIF
    \FOR{ $t=1$ to $T$}	
    \STATE Bandit algorithm pulls arm $a_t$
    \STATE Attacker observes the corresponding reward $r_t = x_{a_t}^\mt\btheta^* + \eta_t$ from the environment
    \IF{ $x_{a_t} \neq \tilde x$}
        \STATE Set $\tilde r_t = x_{a_t}^\mt \tilde\btheta + \tilde \eta_t$  \COMMENT{Attack}
    \ELSE
        \STATE Set $\tilde r_t = r_t$
        \ENDIF
    \STATE Bandit algorithm observes modified reward $\tilde r_t$
    \ENDFOR
    \end{algorithmic}
\end{algorithm}

\begin{proof}[Proof of Theorem \ref{theorem:characterization}]

{\bf Proof of sufficiency.} This direction is relatively straightforward. We show that there exists an efficient attack strategy if CQP \eqref{lp:attackability} holds.

Suppose the attackability index $\epsilon^*>0$ and let $\tilde\btheta_\perp$ be a certificate. In Algorithm~\ref{alg:oracleattack}, we design the \textbf{oracle null space attack}  based on the knowledge of bandit parameter $\btheta^*$.  Let $\tilde\btheta = \btheta_\parallel^* + \tilde \btheta_\perp$ where $\btheta_\parallel^*$ is defined in Eq~\eqref{eq:theta_parallel}. The adversary perturbs the reward of any non-target arm $x_a \not = \tilde{x}$ by $\tilde r_{a, t} = x_a^\mt\tilde\btheta + \tilde \eta_t$, 
whereas the reward of the target arm $\tilde{x}$ is \emph{not} touched. In other words, the adversary misleads the algorithm to believe that $\tilde\btheta$ is the ground-truth parameter We should note both $\tilde\btheta$ and $\btheta^*$ generate the same expected reward on $\tilde{x}$, i.e., $\tilde x^\mt  \tilde\btheta = \tilde x^\mt  \btheta_\parallel^* = \tilde x^\mt  \btheta^*$. To make the attack appear less ``suspicious'', a sub-Gaussian noise term $\tilde \eta_t$ is added to the perturbed reward to make it stochastic. The key idea is that the attacker does not need to perturb the reward of the target arm because the original reward is the same as perturbed reward in expectation. Instead, the attacker only perturbs the reward in the \emph{null space} of $\tilde x$ according to $\tilde\btheta_\perp$, which guarantees the cost-efficiency of the attack. 

Since the perturbed rewards observed by the bandit algorithm are generated by  $\tilde\btheta$, the target arm $\tilde x$ is the optimal arm in this environment due to the attackability index $\epsilon^*$ being strictly positive. According to the definition, any no-regret bandit algorithm will only pull suboptimal arms, i.e., the non-target arms, $o(T)$ times and pull target arm $T-o(T)$ times with high probability. Thus the attack is successful. Moreover,  the cost of oracle attack is $o(T)$ because the attacker only perturbs rewards on the non-target arms for $o(T)$ times, and the cost on each attack is bounded by a constant (because of the finite norm of $x_a$ and $\btheta^*$). Importantly, we note that this argument only relies on the definition of ``no regret'' but does not depend on what the algorithm is. This is crucial for proving the sufficiency of attackability.

\noindent{\bf Proof of necessity.} This is the more difficult direction. Due to space limit, we only provide the proof sketch here while leave the involved technical arguments to Appendix \ref{append:main-proof}.
We shall prove that if $\epsilon^* \leq 0$, the bandit environment is not attackable. To do so, we need to identify at least one no-regret bandit algorithm such that no attack strategy can successfully attack it. We argue that even the vanilla LinUCB is already robust to any attack strategy with $o(T)$ cost when $\epsilon^*\leq 0$.
Recall that LinUCB maintains an estimate  $\hat\btheta_t$ at round $t$ using the attacked rewards $\{\tilde r_{1:t}\}$. We consider LinUCB with the choice of L2-regularization parameter $\lambda$ that guarantees $\Vert\hat\theta_t\Vert_2 < 1$ in order to satisfy the last constraint in CQP \eqref{lp:attackability}. Consider the decomposition $\hat\btheta_{t} = \hat\btheta_{t,\parallel} + \hat\btheta_{t,\perp}$, where $\tilde x \perp \hat\btheta_{t, \perp} $ and $ \tilde x \parallel \hat\btheta_{t,\parallel} $. 

Suppose, for the sake of contradiction, that LinUCB is attackable when $\epsilon^*\leq 0$. According to Definition \ref{def:attackable}, the target arm $\tilde x$ will be pulled $T-o(T)$ times with high probability for infinitely many different time horizons $T$. Fix any large $T$; we know that $\tilde x$ must have the largest UCB score whenever it is pulled at some round $t\in[T]$, or formally, for any $x_a \not = \tilde{x}$ we must have the following:  
\begin{equation}\label{eq:arm_selection}
\tilde x^\mt   \hat\btheta_{t,\parallel} + \text{CB}_t(\tilde x) \geq  x_a^\mt   \hat\btheta_{t,\parallel} + x_a^\mt   \hat\btheta_{t,\perp} + \text{CB}_t(x_a). 
 \end{equation}
By attackability, we know that the above inequality  will hold for infinitely many $t$s. Our main idea to construct the proof is that as $t\to \infty$, we have $\text{CB}_t(\tilde x) \to 0$ and $\text{CB}_t(x_a) > 0$. Moreover, the estimation of $\hat\btheta_{t,\parallel}$ will converge to $\btheta^*_{\parallel}$, since   $\tilde{x}$ will be pulled for $t - o(t)$ times. The key challenge is to show $ \text{CB}_t(x_a) - \text{CB}_t(\tilde x)$, due to Inequality \eqref{eq:arm_selection},  is \emph{strictly} greater than $0$ for all large $t$. To do so, we prove a $\Theta\left(\sqrt{\frac{\log(t/\delta)}{o(t)}}\right)$ \emph{lower} bound for  $\text{CB}_t(x_a)$  (Lemma \ref{lemma:cb_lowerbound} in Appendix \ref{append:main-proof}) and  an $O(\sqrt{\frac{\log(t/\delta)}{t}})$  \emph{upper} bound for  $\text{CB}_t(\tilde x)$  (Lemma \ref{lemma:cb}). The main technical barrier we overcome is the lower bound proof for the confidence bound term, which employs non-standard arguments since most (if not all) of the bandit algorithm analysis only needs the upper bound of the confidence bound terms. Due to this reason, we believe this technical proof is of independent interest, particularly for the analysis of robust properties of linear bandit algorithms.

By letting $t \to \infty$, we obtain the following condition:
\begin{align} 
\quad & \tilde x^\mt    \btheta^*_{\parallel}  > x_a^\mt    \btheta^*_{\parallel}  + x_a^\mt   \hat\btheta_{t,\perp}, \forall   x_a \not = \tilde{x}.
 \end{align}
This implies that for any sufficiently large $t$, there must exist a $\hat\btheta_{t,\perp}$ that 
and makes the optimal objective of CQP \eqref{lp:attackability} $\epsilon^*$ positive. But this contradicts the starting assumption of $\epsilon^*\le0$; hence, the bandit environment is not attackable. 
\end{proof} 

We now provide an intuitive explanation about Theorem \ref{theorem:characterization}. CQP \eqref{lp:attackability} is to find $\tilde\btheta_\perp $ such that: 1) it is orthogonal to $\tilde{x}$ (hence its subscript); and 2) it maximizes the gap $\epsilon$ between  $ \tilde x^\mt    \btheta^*_\parallel$  and the largest $x_a^\mt  ( \btheta^*_\parallel + \tilde\btheta_\perp)$ among all $x_a \not = \tilde{x}$. Theorem \ref{theorem:characterization} states that  the bandit environment is attackable \emph{if and only if}   such a gap 
is strictly larger than $0$, i.e., when the \emph{geometry of context vectors} allows the adversary to lower non-target arms' rewards in the null space of $\tilde x$. 
The constraint $ \Vert \btheta^*_\parallel +  \tilde\btheta_\perp\Vert_2 \leq 1$ ensures the attacked rewards are in the same scale as the unattacked rewards. 

Recent works have shown that any no-regret algorithm for the context-free $k$-armed setting (where arm set $\mathcal{A}$ is orthonormal) can always be attacked \citep{liu2019data}. This finding turns out to be a corollary of Theorem \ref{theorem:characterization}. 

\begin{corollary}\label{cor:k-arm}
\label{cor-MAB}
For standard stochastic multi-armed bandit setting where arm set $\mathcal{A}$ is orthonormal, the environment $\langle \mathcal{A}=\{x_a\}, \btheta^*, \tilde x \rangle$ is  attackable for any target arm $\tilde{x}$. 
\end{corollary}

\begin{proof}
Since arms are orthogonal to each other, it is easy to see that  $\tilde\btheta_\perp = -C[\sum_{x_a: x_a \not = \tilde{x}} x_a]$  achieves objective value $C-  \tilde x^T    \btheta^*_\parallel $ in CQP \eqref{lp:attackability}.  Let $C$ be a large enough positive constant such that the objective value is positive gives us a feasible $\tilde\btheta_\perp$  to CQP \eqref{lp:attackability}, which yields the corollary.
\end{proof}

The intuition behind this corollary is that arms in context-free stochastic multi-armed bandits are independent, and an adversary can arbitrarily lower the rewards of non-target arms to make the target arm optimal. This is also the attack strategy in~\citep{jun2018adversarial,liu2019data}.   \citet{garcelon2020adversarial} showed that similar idea works for $k$-armed linear contextual bandits where each arm is associated with an unknown bandit parameter and the reward estimations are independent among different arms. 

We should point an important distinction between poisoning attacks to $k$-armed bandits and another line of research on \emph{stochastic bandits under adversarial corruption} initiated by \citet{lykouris2018stochastic}. For poisoning attacks considered in this paper, the adversary manipulates the realized rewards \emph{after} the algorithm selects an action, whereas in \citep{lykouris2018stochastic}, the adversary manipulates the entire reward vector \emph{before} the algorithm selects any action. Obviously, the later threat model is strictly weaker and has led to various bandit algorithms that can have sublinear regret so long as the total manipulation is sublinear in $T$  \citep{lykouris2018stochastic,zimmert2021tsallis}. 

%% file: real.tex
\section{Effective Attacks without  Knowledge of True Model Parameters}\label{sec:twostage}

In the previous section, we characterized the attackability of a linear stochastic bandit environment by assuming the knowledge of ground-truth bandit parameter $\btheta^*$. We now show that such prior knowledge is not needed when designing practical attacks. Towards this end, we propose provably effective attacks against two representative bandit algorithms: the most well-known LinUCB and a state-of-the-art robust linear stochastic bandit algorithm, Robust Phase Elimination \citep{bogunovic2021stochastic}. We remark that the optimal attacks to these algorithms depend on the characteristics of algorithms themselves and are generally different, due to their different levels of robustness. 
This also resonates the important message mentioned in the introduction, i.e., the attackability analysis often goes hand-in-hand with the understanding of robustness of different algorithms, as reflected in various parts of our analysis.  
However, we point out that it is an intriguing open question to understand whether there is a single attack strategy that can manipulate any no-regret algorithm in an attackable environment.

\vspace{-1mm}
\paragraph{Two-stage Null Space Attack.}  
Our proposed attack method is presented in Algorithm~\ref{alg:attack}. The method spends the first $T_1$ rounds as the first stage to attack rewards on all arms by imitating a bandit environment $\btheta_0$, in which $\tilde{x}$ is the best arm such that arm $\tilde{x}$ will be pulled most often by the bandit algorithm. This stage is for the adversary to observe the true rewards of $\tilde{x}$ from the environment, which helps it estimate the parameter $\btheta^*_{\parallel}$. At round $T_1$, the method uses the estimate of $\btheta^*_{\parallel}$, denoted as $\tilde\btheta_{\parallel}$, to perform the ``attackability test'' by solving CQP \eqref{lp:attackability} using  $\tilde\btheta_{\parallel}$ to obtain an estimated index $\tilde \epsilon^*$ and certificate $\tilde\btheta_{\perp}$. 
If $\tilde \epsilon^* > 0$, the method asserts the environment is attackable and starts the second stage of attack. From $T_1$ to $T$, the method perturbs the reward of \emph{non-target arms}  by $\tilde{r}_t = x_{a_t}^\mt  (\tilde\btheta_{\parallel} + \tilde\btheta_{\perp}) + \tilde \eta_t$ (just like the oracle attack but using the estimated $\tilde\btheta_{\parallel}$). When the bandit algorithm pulls the target arm $\tilde x$ for the first time in the second stage, the method will compensate the reward of $\tilde x$ as shown in line 20, where $n(\tilde x)$ is the number of times target arm is pulled before $T_1$. The goal is to correct the rewards on $\tilde x$ collected in the first stage to follow $\tilde\btheta$ instead of $\btheta_0$. 
Note that a larger $T_1$ brings in more observations on $\tilde{x}$ to improve the estimate of $\btheta^*_\parallel$; but it also means a higher attack cost. The optimal choice of $T_1$ depends on the ``robustness'' of the bandit algorithm to be attacked. Consequently, it also leads to different attack cost for different algorithms. 
For example, as we will show next,  the attack to Robust Phase Elimination will be more costly than the attack to LinUCB.

\begin{algorithm}[t]
    \caption{Two-stage Null Space Attack} \label{alg:attack}
    \begin{algorithmic}[1]
    \STATE \textbf{Inputs:} $T, T_1$
    \STATE $\btheta_0 = \arg \max_{\Vert\btheta\Vert_2\leq 1} \Big[ \tilde x^\mt   \btheta - \max_{x_a \not = \tilde{x}} x_a^\mt  \btheta \Big]$, let $\epsilon^*_0$ be its optimal objective 
    \IF[Initial attackability test]{$\epsilon^*_0 \leq 0$}   
      \STATE   \textbf{return} Not attackable
    \ENDIF
    \FOR{ $t=1$ to $T_1$}	
        \STATE Set $\tilde r_t = x_{a_t}^\mt   \btheta_0 + \tilde\eta_t$
        \COMMENT{Attack as if $\tilde x$ is the best}
        \STATE Bandit algorithm observes modified reward $\tilde r_t$
    \ENDFOR

        \STATE Estimate $\tilde\btheta_\parallel =   \frac{\sum_{i=1}^{n(\tilde x)} r_i(\tilde x)}{n(\tilde x)\Vert\tilde x\Vert_2^2}\tilde x $  
        \STATE Solve CQP \eqref{lp:attackability} using $\tilde\btheta_\parallel$ to obtain the estimated attackability index $\tilde\epsilon^*$ and certificate $\tilde{\btheta}_{\perp}$ 
        \IF[Attackability test]{$\tilde\epsilon^*\leq 0$} 
        \STATE \textbf{return} Not attackable
    \ELSE[Attack stage] 
        \STATE Set $\tilde\btheta = \tilde\btheta_\parallel+\tilde\btheta_\perp$
    \FOR{$t=T_1+1$ to $T$ }
    \IF{ $x_{a_t} \neq \tilde x$}  
            \STATE Set $\tilde r_t = x_{a_t}^\mt \tilde\btheta + \tilde\eta_t$
        \ELSIF{$x_{a_t} = \tilde x$ for the first time}
        \STATE Set $\tilde r_t = n(\tilde x) \times \tilde x^\mt  (\tilde\btheta - \btheta_0) + \tilde x^\mt \tilde\btheta + \tilde\eta_t$ 
        \ELSE
        \STATE Set $\tilde r_t = r_t$
        \ENDIF
        \STATE Bandit algorithm observes modified reward $\tilde r_t$ 
    \ENDFOR
    \ENDIF
    \end{algorithmic}
\end{algorithm}

Note that our attackability test might make both false positive and false negative assertions due to the estimation error in  $\tilde\btheta_\parallel$. But as $T$ becomes larger, the estimate gets more accurate and the assertion is correct with   high probability.

\begin{remark} We acknowledge that the rewards from the two stages follow different reward distributions and could be detected, e.g., using some homogeneity test \citep{li2021unifying}. Thus a bandit player could realize the attack if equipped with some change detector. However, attacking such a cautious bandit algorithm is beyond the scope of this paper.  
Moreover, it is very difficult (if not impossible) to attack with a stationary reward distribution or undetectable perturbations (e.g., using a fixed target parameter $\tilde\theta$). We could easily find cases where the adversary's parameter $\tilde \theta$ is limited to extremely few choices and it is almost impossible to directly start the attack with a valid $\tilde \theta$ without knowing $\theta^*$. For example, if we change the third arm in Example 1 to $x_3 = (-1+\epsilon, 0)$ with a small $\epsilon$, we can see that the valid parameters are only in a small range around $\tilde \theta = (-1-\epsilon, 1)$. Therefore, in order to attack with a stationary reward distribution, the adversary needs to start from somewhere very close to $\tilde \theta = (-1-\epsilon, 1)$, which we believe is extremely difficult without knowing $\theta^*$. Overall, we think designing an attack strategy against a bandit algorithm with reward change detector or showing the inability to attack such cautious algorithms is an important future work of ours. 
\end{remark}

\paragraph{Attack against LinUCB.}  We now show how LinUCB algorithm can be attacked by Algorithm~\ref{alg:attack}.


\begin{theorem}\label{theorem:attackLinUCB}
For large enough $T_1$,  the attack strategy in Algorithm \ref{alg:attack}  will correctly assert the attackability with probability at least $1-\delta$. Moreover, when the environment is attackable, with probability at least $1-3\delta$ the attack strategy  will fool LinUCB to pull  non-target arms at most
\begin{equation*}
\begin{split}
O\Big(&d\big(\sqrt{\log(T/\delta)} + \sqrt{T_1}\log{(T_1/\delta)} \\&+ \sqrt{T\log(1/\delta)}/\sqrt{T_1}\big)\sqrt{T\log (T/\delta)}/\epsilon^*\Big)       
\end{split}
\end{equation*} 
rounds. And with probability at least $1-4\delta$, the adversary's cost is at most
\begin{equation*}
\begin{split}
O\Big(T_1 + &d\big(\sqrt{\log(T/\delta)} + \sqrt{T_1}\log{(T_1/\delta)} \\&+ \sqrt{T\log(1/\delta)}/\sqrt{T_1}\big)\sqrt{T\log (T/\delta)}/\epsilon^*\Big).
\end{split}
\end{equation*}
Specifically, when $T_1 = T^{1/2}$, the strategy gives the minimum attack cost in the order of $\tilde O(T^{3/4})$, 
and the non-target arms are pulled at most $\tilde O(T^{3/4})$ rounds. 
\end{theorem}

\begin{proof}[Proof Sketch.]
To prove the the assertion is correct with high probability, the key idea is that the estimated $\tilde\btheta_\parallel$ is close to the true parameter $\btheta^*_\parallel$.
We first note that in the first stage, the bandit algorithm will pull the target arm $\tilde x$ $T_1 -O(\sqrt{T_1})$ times, since $\tilde x$ is the best arm according to $\btheta_0$. According to the Hoeffding's inequality, the estimation error  
$\Vert\tilde\btheta_\parallel -  \btheta^*_\parallel \Vert_2 \leq \sqrt{\frac{2\log(2/\delta)}{T_1 -O(\sqrt{T_1})}}$. 
Therefore, with a large enough $T_1$, the error on $\tilde x$'s reward estimation is smaller than $\epsilon^*$. Thus solving CQP \eqref{lp:attackability} with $\tilde\btheta_\parallel$ and we can correctly assert attackability with positive estimated index $\tilde\epsilon^*$ when the environment is attackable with index $\epsilon^*$. 

To prove the success and the cost of the attack, we need to analyze the behavior of LinUCB under the reward discrepancy  between the two stages.
Our proof crucially hinges on the following robustness result of LinUCB. 

\begin{lemma}\label{lemma:robustrness_ridgereg}
Consider LinUCB with ridge regression under poisoning attack. Let $S'_t = \sum_{\tau\in\{1 \dots t\}, x_{a_\tau}\neq\tilde x} |\Delta_\tau|$ be the total corruption on non-target arms until time $t$ and assume every corruption on target arm is bounded by $\gamma$. For any $t= 1 \dots T$, with probability at least $1-\delta$ we have 
 \begin{equation}
    \Vert \tilde\btheta- \hat{\btheta}_t\Vert _{A_t} \leq \alpha_t + S'_t/\sqrt{\lambda} + \gamma\sqrt{t}
 \end{equation} 
 where $\alpha_t = \sqrt{d\log\left(\frac{1+t/\lambda}{\delta}\right)} + \sqrt{\lambda}$.  
\end{lemma}

Based on this lemma, we can derive the regret $R_T(\tilde\btheta)$ of LinUCB with $\tilde\btheta$ as the true parameter. The total corruption on non-target arms is $O\big(d\sqrt{T_1}\log{(T_1/\delta)}\big)$ given the rewards are generated by $\btheta_0$ (the rewards of target arm in the first stage are compensated in line 20). Because the target arm's rewards are not attacked in the second stage and follows $\btheta^*$, we have $\gamma = \tilde O(1/\sqrt{T_1})$.  Since the non-target arms are pulled at most $R_T(\tilde\btheta)/\epsilon^*$ rounds, substitute the total corruption back and we have the resulting bound. 

The attack cost has two sources: attacks in the first stage for $T_1$ times, and attacks on the non-target arms in the second stage. The second term has the same order as the number of rounds where the non-target arms are pulled by LinUCB. Each attack cost can be decomposed as 1) the change of mean reward $|x_a^\mt (\tilde\btheta - \btheta^*)|$, and 2) the sub-Gaussian noise $|\tilde \eta_t|$, the sum of which increases linearly with high probability. By setting $T_1 = T^{1/2}$, the total cost achieves $\tilde O(T^{3/4})$.
\end{proof}

\begin{remark}
Lemma \ref{lemma:robustrness_ridgereg} shows that LinUCB still enjoys sublinear regret for any corruption amount $S = o(\sqrt{T})$. This tolerance of $o(\sqrt{T})$ attack turns out to be the same as the recently proposed robust linear contextual bandit algorithm based on phase elimination in \citep{bogunovic2021stochastic} (which we examine   next). However, the regret term $S \sqrt{T}$ in LinUCB has a worse dependence  on $S$  within the $S = o(\sqrt{T})$ regime   compared to the $O(S^2)$ regret dependence of the robust algorithm in \citep{bogunovic2021stochastic}. 
\end{remark} 

\paragraph{Attack against Robust Phase Elimination.}   We now show that Robust Phase Elimination (RobustPhE) \citep{bogunovic2021stochastic} can  also be attacked by Algorithm~\ref{alg:attack}.  Comparing to attacking LinUCB, the robustness of RobustPhE brings challenge to the first stage attack, as the attack cost is more sensitive to the length of this stage. 

\begin{corollary}\label{corollary:attackrobustPhE}

For any large enough $T_1$,  the attack strategy in Algorithm \ref{alg:attack}  will correctly assert the attackability with high probability. Moreover, when the environment is attackable, with probability at least $1-2\delta$ the attack strategy  will fool RobustPhE to pull  non-target arms at most
\[O\Big(\big(d\sqrt{T}\log (T/\delta) + \sqrt{d}T\log(T)\log (1/\delta)/\sqrt{T_1} + T_1^2\big)/\epsilon^*\Big)\]
\normalsize
rounds. And with probability at least $1-3\delta$, the adversary's cost is at most
\[
O\Big(T_1 + \big(d\sqrt{T}\log (T/\delta) + \sqrt{d}T\log(T)\log (1/\delta)/\sqrt{T_1} + T_1^2\big)/\epsilon^*\Big)
\]
\normalsize
Specifically, setting $T_1 = T^{2/5}$ gives the minimum attack cost order $\tilde O(T^{4/5})$ and the non-target arms are pulled at most $\tilde O(T^{4/5})$ rounds. 
\end{corollary}

RobustPhE has an additional regret term $O(S^2)$ for total corruption $S$ (assuming $S$ is unknown to the bandit algorithm). If we view the second stage attack model $\tilde\btheta$ as the underlying environment bandit model, rewards generated by $\btheta_0$ in the first stage should be considered as corrupted rewards. The $T_1$ amount of rewards from the first stage means $T_1$ amount of corruption, which leads to the additional $T_1^2$ term 
compared with the results in Theorem~\ref{theorem:attackLinUCB}. Hence, the adversary can only run fewer iterations in the first stage but spend more attack cost there.
On the other hand, this also facilitates the design of attack such that line 19-20 in Algorithm \ref{alg:attack} is not necessary: the corruption in the first stage can be handled by the robustness of RobustPhE. The unattacked rewards in second stage are viewed as misspecification from $\tilde\btheta$ with error $\gamma$, which leads to the $\tilde O(\gamma T)$ term (the second term) in the bound. Our success in attacking RobustPhE does not violate the robustness claim in the original paper~\citep{bogunovic2021stochastic}: RobustPhE could tolerate $
o(\sqrt{T})$ corruption and our attack cost is $\tilde O(T^{4/5})$.

%% file: exp.tex
\section{Experiments}\label{sec:exp}

\begin{figure}[t]
    \centering
    \begin{tabular}{c c}
     \includegraphics[width=7.5cm]{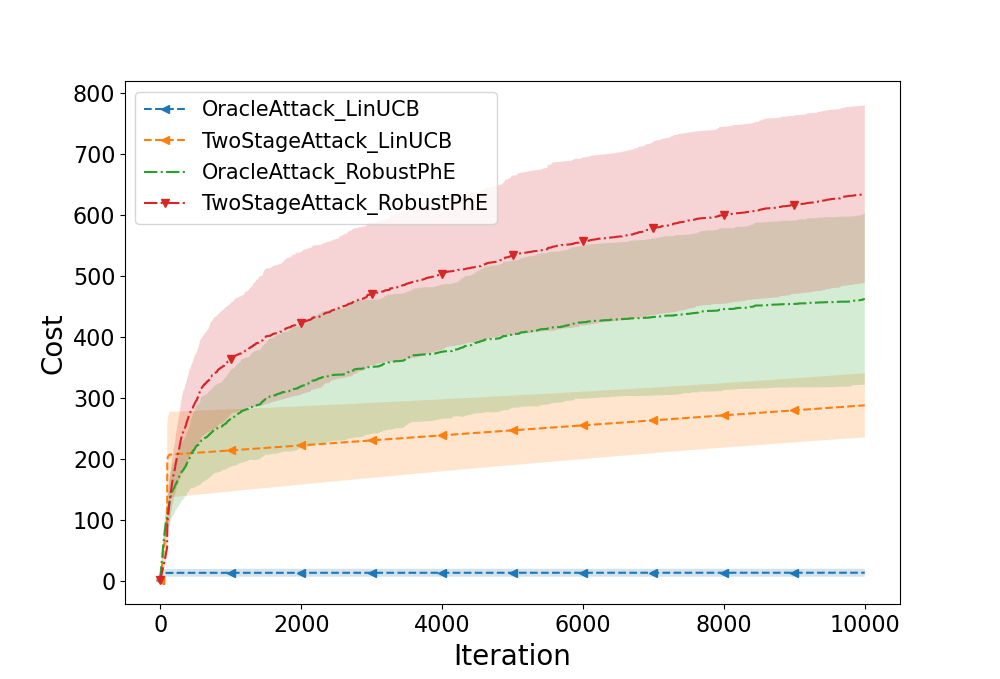} & 
     \includegraphics[width=7.5cm]{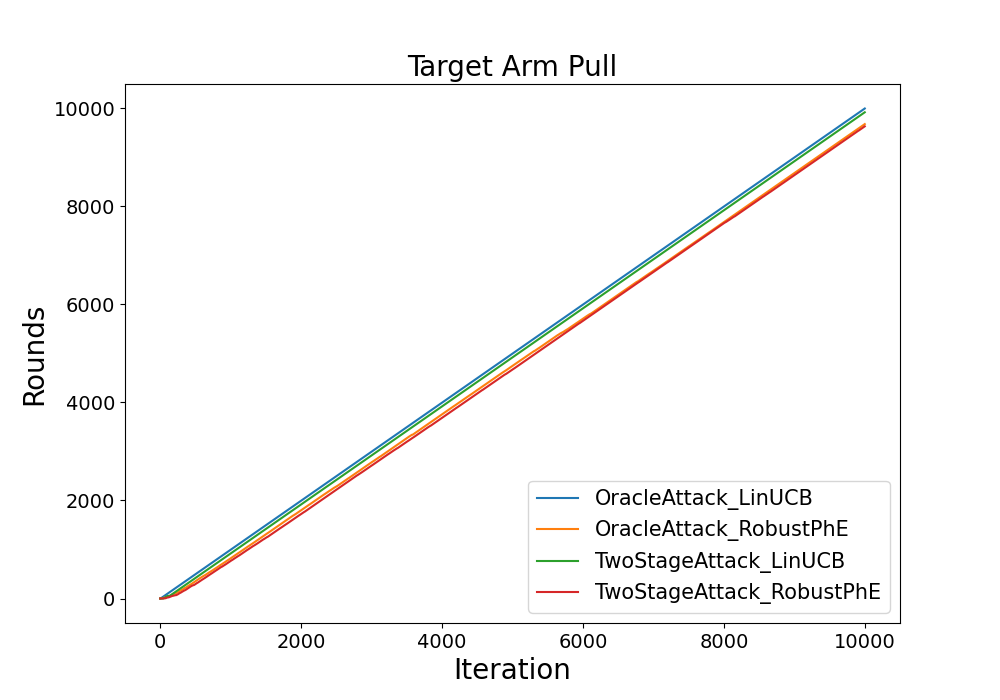}
     \\
    (a) Total cost of attack. & (b) Target arm pulls.
    \end{tabular}
    \caption{Total cost and target arm pulls under different attack methods. We report averaged cost and variance of 10 runs. }\label{fig:simu}
\end{figure}

We use simulation-based experiments to validate the effectiveness and cost-efficiency of our proposed attack methods.  
In our simulations, we generate a size-$k$ arm pool $\cA$, in which each arm $a$ is associated with a context vector $x_a \in \mathbb{R}^d$. Each dimension of $x_a$ is drawn from a set of zero-mean Gaussian distributions with variances sampled from a uniform distribution $U(0,1)$. Each $x_a$ is then normalized to $\Vert x_a\Vert_2 = 1$.
The bandit model parameter $\btheta^*$ is sampled from $N(0, 1)$ and normalized to $\Vert\btheta^*\Vert_2  = 1$. We set $d$ to 10, the standard derivation $\sigma$ of Gaussian noise $\eta_t$ and $\tilde\eta_t$ to 0.1, and the arm pool size $k$ to 30 in our simulations. We run the experiment for $T=10,000$ iterations. To create an attackable environment, we will re-sample the environment $\langle \cA, \btheta^*, \tilde x \rangle$ until it is attackable, following Theorem~\ref{theorem:characterization}. 

We compared the two proposed attack methods, oracle null space attack and two-stage null space attack, against LinUCB~\cite{li2010contextual} and Robust Phase Elimination (RobustPhE)~\cite{bogunovic2021stochastic}. We report average results of 10 runs where in each run we randomly sample an attackable environment.  Both oracle attack and two-stage attack can effectively fool the two bandit algorithms to pull the target arm linear times as shown in Figure~\ref{fig:simu}(b), and the total cost of the attack is shown in
Figure~\ref{fig:simu}(a). We observe that both attack methods are cost-efficient with a sublinear total cost, while the two-stage attack requires higher attack cost when attacking the same bandit algorithm. Specifically, we notice that the adversary spends almost \emph{linear} cost in  the first stage. This is because in  the first stage the adversary  attacks according to parameter $\btheta_0$, which leads to an almost constant cost at every round. This is to help the adversary to estimate bandit model parameter in order to construct target parameter $\tilde\btheta$. In the second stage, the cost gets much smaller since the adversary only attacks the non-target arms. 
We also notice that for the same attack method, attacking RobustPhE requires a higher cost and the number of target arm pull is also smaller comparing with attacking LinUCB, due to the robustness of the algorithm.



%% file: related.tex
\section{Related Work}
Adversarial attacks to  bandit algorithms was first studied in the stochastic multi-armed bandit setting~\citep{jun2018adversarial, liu2019data} and recently in linear contextual bandits  \citep{garcelon2020adversarial}.  
These works share a similar attack idea:  lowering the rewards of non-target arms while not modifying the reward of target arm. However, as our attackability analysis revealed, this idea can fail in a linear stochastic bandit environment where one cannot lower the rewards of non-target arms without modifying the reward of target arm, due to their correlation. This insight is a key reason that gives rise to unattackable environments.   \citet{ma2018data} also considered the attackability issue of linear bandits, but under the setting of \emph{offline} data poisoning attack where the adversary has the power to modify the rewards in history. 
There are also several recent works on reward poisoning attacks against reinforcement learning~\citep{yu2019efficient,zhang2020adaptive,rakhsha2021reward,rakhsha2020policy}, but with quite different focus as ours. 
Besides reward poisoning attacks, recent works also studied other threat model such as action poisoning attacks~\citep{liu2020action,liu2021provably}.


A parallel line of works focused on improving the robustness of bandit algorithms.   \citet{lykouris2018stochastic} proposed a robust MAB algorithm and  \citet{gupta2019better} further improved the solution with additive regret dependency on attack cost. \citet{zimmert2021tsallis, masoudian2021improved} proposed best-of-both-world solutions for both stochastic and adversarial bandits which also solved stochastic bandits with adversarial corruption.  \citet{ito2021optimal} further proposed optimal robust algorithm to adversarial corruption. 
These work assumed a weaker oblivious adversary who determines the manipulation before the bandit algorithm pulls an arm.   \citet{hajiesmaili2020adversarial} studied robust  adversarial bandit algorithm.  \citet{guan2020robust} proposed a median-based robust bandit algorithm for probabilistic unbounded attack.  \citet{bogunovic2021stochastic} proposed robust phase elimination algorithm for linear stochastic bandits under a stronger adversary (same as ours), which could tolerate $o(\sqrt{T})$ corruption when the total corruption is unknown to the algorithm. We showed that our two-stage null space attack could effectively attack this algorithm with $\tilde O(T^{4/5})$  cost.
Recently, \citet{zhao2021linear} proposed an OFUL style robust algorithm that can handle infinite action set, but only tolerates $o(T^{1/4})$ corruption.

%% file: supp.tex
\newpage
\appendix
\onecolumn

\section{Notations}
For clarity, we collect the notations used in the paper below.

\begin{tabular}{l|l}
    $\tilde x$ & Context vector of target arm\\
    $x_a$ & Context vector of arm $a$\\
    $\btheta^*$ & Unknown bandit model parameter of the environment\\
    $\btheta^*_{\parallel}$ & Projection of $\btheta^*$ on $\tilde x$\\
    $r_t$ & Unattacked reward feedback at time $t$\\
    $\eta_t$ & Sub-Gaussian noise in reward, i.e., $r_t = x_t^\mt\btheta^* + \eta_t$.\\
    $\tilde r_t$ & Attacked reward\\
    $\hat\btheta_t$ &  Parameter estimated by the victim bandit algorithm with attacked rewards $\{\tilde r_{1:t}\}$\\
    $\tilde\btheta_\parallel$ & Parameter parallel to $\tilde x$, estimated by adversary with unattacked rewards\\
    $\tilde \btheta$ & Paramter of adversary's attack strategy\\
    $\tilde \btheta_\perp$ & Attackability certificate, the parameter orthogonal to $\tilde x$ solved by CQP \eqref{lp:attackability}\\
    $\epsilon^*$& Attackability index, optimal objective of CQP \eqref{lp:attackability}\\
    $\tilde\epsilon^*$& Estimated index, optimal objective of CQP \eqref{lp:attackability} with $\tilde\btheta_\parallel$ replacing $\btheta^*_\parallel$\\
    
\end{tabular}

\section{Details on Attackability of Linear Stochastic Bandits}\label{append:main-proof}

\subsection{Necessity Proof of Theorem \ref{theorem:characterization} }

To prove its necessity, we will rely on the following results.

\begin{lemma}\label{lemma:cb}
Suppose arm $x$ is pulled $n$ times till round $t$ by LinUCB. Its confidence bound $\text{CB}_t(x)$ in LinUCB satisfies
\begin{equation}\label{eq:x_Anorm}
\text{CB}_t(x) \leq \frac{\alpha_t}{\sqrt{n}}.
\end{equation}
with probability at least $1-\delta$, where $\alpha_t = \sqrt{d\log\left(\frac{1+t/\lambda}{\delta}\right)} + \sqrt{\lambda}$. Furthermore, we have
\begin{equation}
\text{CB}_t(x) \leq O\left(\sqrt{\frac{ \log (t/\delta)}{n}}\right)    
\end{equation}
with probability at least $1-\delta$.
\end{lemma}
\begin{proof}
In \cite{Improved_Algorithm}, the exploration bonus term is computed as $\text{CB}_t(x) = \alpha_t \lVert x \rVert_{\bA_t^{-1}}$.
Denote $\bA'_t = n\times xx^\mt + \lambda\bI$. Since $\bA_t = \sum_{i=1}^t x_{a_i} x_{a_i}^\mt + \lambda\bI$, we have  $\bA_t \succ \bA'_t$. We can thus bound $\lVert x \rVert_{\bA_t^{-1}}$ by
\begin{equation}\label{eq:x_Ainv}
    \lVert x \rVert_{\bA_t^{-1}} \leq \lVert x \rVert_{{\bA'_t}^{-1}} \leq \frac{1}{\sqrt{n}}
\end{equation}

According to Theorem 2 in ~\cite{Improved_Algorithm},  
\begin{equation}
 \alpha_t = \sqrt{d\log\left(\frac{1+t/\lambda}{\delta}\right)} + \sqrt{\lambda} = \Theta(\sqrt{\log (t/\delta)}).   
\end{equation}
Combining $\alpha_t$ and \eqref{eq:x_Ainv}  finishes the proof.
\end{proof}
\begin{claim}\label{claim:cb}
Target arm $\tilde x$ is pulled $n = T-o(T)$ times till round $T$. According to Lemma~\ref{lemma:cb}, we have
\begin{equation}
\text{CB}_t(\tilde x) \leq O\left(\sqrt{\frac{\log (T/\delta)}{T-o(T)}}\right)    
\end{equation}
\end{claim}
\begin{lemma}\label{lemma:cb_lowerbound}
Suppose arm $x$ is pulled $t-m$ times till round $t$ by LinUCB, and other arms are pulled $m$ times in total. Confidence bound $\text{CB}_t(x_a)$ of any arm $x_a$ that is not parallel to $x$ satisfies
\begin{equation}\label{eq:x_Anorm_lowerbound}
\text{CB}_t(x_a) \geq \alpha_t \left(\frac{1}{\sqrt{m +\lambda}} -\frac{b}{\sqrt{t-m+\lambda}}\right)
\end{equation}
with probability at least $1-\delta$, where $\alpha_t = \sqrt{d\log\left(\frac{1+t/\lambda}{\delta}\right)} + \sqrt{\lambda}S$ and  constant $b = \frac{x_a^\mt x }{x^\mt x}$. Furthermore, we have
\begin{equation}
\text{CB}_t(x_a) \geq \Theta\left(\sqrt{\log (t/\delta)} \left(\frac{1}{\sqrt{m}} -\frac{1}{\sqrt{t-m}}\right)\right)
\end{equation}
with probability at least $1-\delta$
\end{lemma}
\begin{proof}
Since $\text{CB}_t(x) = \alpha_t \lVert x \rVert_{\bA_t^{-1}}$, we need to show a lower bound of $\lVert x_a \rVert_{\bA_t^{-1}}$. Since $x_a \nparallel x$, we decompose $x_{a} = x_{a}^{\parallel} + x_{a}^{\perp}$, where $x_{a}^{\parallel} \parallel x$. By the reverse triangle inequality we have
\begin{align}
    \lVert x_a \rVert_{\bA_t^{-1}}&\geq \lVert x_a^\perp \rVert_{\bA_t^{-1}} - \lVert x_a^\parallel \rVert_{\bA_t^{-1}}
\end{align}

First we analyze the term $\lVert x_a^\perp \rVert_{\bA_t^{-1}}$. Decompose $\bA_t = (t-m)\times xx^\mt + \sum_{i, x_{a_i} \neq x} x_{a_i}x_{a_i}^\mt + \lambda\bI$ and let $\bA'_t = \sum_{i, x_{a_i} \neq x} x_{a_i}x_{a_i}^\mt + \lambda\bI$. Since $x_a^\perp \perp x$, we have \[{x_a^\perp}^\mt \bA_t x_a^\perp = {x_a^\perp}^\mt\bA'_t x_a^\perp.\] There are at most $m$ terms in the summation of $\bA'_t = \sum_{i, x_{a_i} \neq x} x_{a_i}x_{a_i}^\mt + \lambda I$, thus
\[
{x_a^\perp}^\mt\bA'_t x_a^\perp \leq {x_a^\perp}^\mt\left( \frac{m}{\Vert{x_a^\perp}\Vert_2^2}\times {x_a^\perp}{x_a^\perp}^\mt + \lambda\bI \right) x_a^\perp \leq m + \lambda
\]
Then we have
\begin{equation}\label{eq:x_Ainv_lb}
    \lVert x_a^\perp \rVert_{\bA_t^{-1}} = \sqrt{{x_a^\perp}^\mt\bA^{-1}_t x_a^\perp} = \sqrt{{x_a^\perp}^\mt\bA'^{-1}_t x_a^\perp} = \frac{1}{\sqrt{{x_a^\perp}^\mt\bA'_t x_a^\perp}} \geq \frac{1}{\sqrt{m + \lambda}}
\end{equation}

Similar to Eq~\eqref{eq:x_Ainv}, we have
\begin{equation}
    \lVert x_a^\parallel \rVert_{\bA_t^{-1}} \leq  \frac{\lVert x_a^\parallel \rVert_2}{\lVert x\rVert_2}\frac{1}{\sqrt{t-m+\lambda}}
\end{equation}
Let constant $b = \frac{\lVert x_a^\parallel \rVert_2}{\lVert x\rVert_2} = \frac{x_a^\mt x }{x^\mt x}$.
Substitute the terms and we have
\begin{align}
    \text{CB}_t(x) = \alpha_t \lVert x \rVert_{\bA_t^{-1}} \geq \alpha_t \left(\lVert x_a^\perp \rVert_{\bA_t^{-1}} - \lVert x_a^\parallel \rVert_{\bA_t^{-1}}\right) \geq \alpha_t \left(\frac{1}{\sqrt{m + \lambda}} -\frac{b}{\sqrt{t-m+\lambda}}\right).
\end{align}
\end{proof}
\begin{claim}\label{claim:cb_lowerbound}
Non-target arms are pulled $m = o(T)$ times till round $T$. According to Lemma~\ref{lemma:cb_lowerbound}, any arm $x_a \nparallel \tilde x$ satisfies
\begin{equation}
\text{CB}_t(x_a) \geq \Theta\left(\sqrt{\log (T/\delta)} \left(\frac{1}{\sqrt{o(T)}} -\frac{1}{\sqrt{T-o(T)}}\right)\right)   
\end{equation}
with probability at least $1-\delta$.
\end{claim}

\begin{lemma}\label{claim:theta_parallel}
Suppose the non-target arms $\{x_a \neq \tilde x\}$ are pulled $o(T)$ times, the arm $\tilde x$ is pulled $T-o(T)$ times, and the total manipulation is $C_T$. With probability at least $1-\delta$, reward estimation error satisfies 
\begin{equation}
\left\vert x^\mt \hat{\btheta}_{T,\parallel} - x^\mt {\btheta}^*_\parallel \right\vert\leq \frac{C_T}{T - o(T)} + \frac{\alpha_t}{\sqrt{T - o(T)}}, \qquad \forall x \in \mathcal{A}.
\end{equation}

\end{lemma}
\begin{proof}

\begin{align*}
 \Vert\hat\btheta_{T,\parallel} - \btheta^*_\parallel\Vert_2  &= \left\Vert\frac{\tilde x^\mt (\hat\btheta_T - \btheta^*)}{\lVert\tilde x\rVert^2_2}\tilde x\right\Vert_2\\
 &=\frac{1}{{\lVert\tilde x\rVert^2_2}}\left\Vert\tilde x^\mt \bA_t^{-1}\left( \sum_{t=1}^T x_t\tilde r_t(x_t) -  \bA_t\btheta^*\right)\tilde x\right\Vert_2 \\
&=\frac{1}{{\lVert\tilde x\rVert^2_2}}\left\Vert\tilde x^\mt \bA_t^{-1}\left( \sum_{t=1}^T x_t(\tilde r_t(x_t) -  x_t^\mt\btheta^*)-\lambda\btheta^*\right)\tilde x\right\Vert_2 \\
&\leq \frac{1}{{\lVert\tilde x\rVert^2_2}}\left\Vert\tilde x^\mt \bA_t^{-1}\left(\sum_{t=1}^T x_t\Delta_t+\sum_{t=1}^T x_t\eta_t -\lambda\btheta^*\right)\tilde x\right\Vert_2\\
&\leq \frac{1}{{\lVert\tilde x\rVert^2_2}}\left\Vert\tilde x^\mt \bA_t^{-1} \left(\sum_{t=1}^T x_t\Delta_t\right)\tilde x\right\Vert_2+ \frac{1}{{\lVert\tilde x\rVert^2_2}} \left\Vert\tilde x^\mt \bA_t^{-1/2} \bA_t^{-1/2} \left(\sum_{t=1}^T x_t\eta_t-\lambda\btheta^*\right)\tilde x\right\Vert_2\\
&\leq \frac{C_T}{T-o(T)} + \frac{\alpha_t}{\sqrt{T - o(T)}}
\end{align*} 
Note that $\lVert\tilde x\rVert_2 \leq 1$. In the last step, the first term is because there are $T-o(T)$ number of $\tilde x \tilde x^\mt$ in $A_t$ and $\Vert\tilde x^\mt \bA_t^{-1}\Vert_2\leq \frac{1}{T-o(T)}$, and $ \Vert\sum_{t=1}^T x_t\Delta_t\Vert_2 $ is bounded by total manipulation $C_T$. Similarly, in the second term we have $\Vert\tilde x^\mt \bA_t^{-1/2}\Vert_2\leq \frac{1}{\sqrt{T-o(T)}}$, and  $\Vert\bA_t^{-1/2} \left(\sum_{t=1}^T x_t\eta_t-\lambda\btheta^*\right)\Vert_2 \leq \Vert\bA_t^{-1/2} \left(\sum_{t=1}^T x_t\eta_t \right)\Vert_2 + \Vert\bA_t^{-1/2}  \lambda\btheta^* \Vert_2 = \Vert \sum_{t=1}^T x_t\eta_t \Vert_{\bA_t^{-1}} + \Vert  \lambda\btheta^*\Vert_{\bA_t^{-1}} \leq \alpha_t$  is the self-normalized bound for vector-valued martingales following Theorem 1 in ~\cite{Improved_Algorithm}. 
\end{proof}

Now we are ready to prove that the index $\epsilon^*$ in CQP \eqref{lp:attackability} being positive is the necessary condition of an attackable environment. 

\begin{proof} [Proof of Necessity of Theorem \ref{theorem:characterization}]
We prove if $\epsilon^* \leq 0$,  the bandit environment is not attackable. To prove this, we show that there exists some no-regret bandit algorithm (LinUCB in particular) such that no attacking strategy can succeed. In particular, we will show that LinUCB (with a specific choice of its L2-regularization parameter $\lambda$) is robust under any attacking strategy with $o(T)$ cost when $\epsilon^*\leq 0$. 
We prove it by contradiction: assume LinUCB is attackable with $o(T)$ cost when $\epsilon^*\leq 0$. According to Definition \ref{def:attackable}, the target arm $\tilde x$ will be pulled $T-o(T)$ times for infinitely many different time horizons $T$, and the following inequalities hold when arm $\tilde x$ is pulled by LinUCB:
\begin{align}\label{eq:arm_selection_supp}
\quad & \tilde x^\mt   \hat\btheta_{T,\parallel} + \text{CB}_T(\tilde x) >  x_a^\mt   \hat\btheta_{T,\parallel} + x_a^\mt   \hat\btheta_{T,\perp} + \text{CB}_T(x_a), \forall   x_a \not = \tilde{x}
 \end{align}
where  $\hat\btheta_t$ is LinUCB's estimated parameter at round $t$ based on the attacked rewards. We decompose $\hat\btheta_T = \hat\btheta_{T,\parallel} + \hat\btheta_{T,\perp}$, where $\tilde x \perp \hat\btheta_{t, \perp} $ and $ \tilde x \parallel \hat\btheta_{T,\parallel} $.
We will now show that the above inequalities lead to
\begin{align*} 
\quad & \tilde x^\mt    \btheta^*_{\parallel}    >  x_a^\mt    \btheta^*_{\parallel}  + x_a^\mt   \hat\btheta_{T,\perp}, \forall   x_a \not = \tilde{x}
 \end{align*}
when $T \to \infty$.


By Lemma~\ref{claim:theta_parallel} we have
\begin{align*}
 x_a^\mt \hat{\btheta}_{T,\parallel}  &\geq x_a^\mt {\btheta}^*_\parallel - \frac{C_T}{T - o(T)} - \frac{\alpha_T}{\sqrt{T - o(T)}}\\
\tilde x^\mt \hat{\btheta}_{T,\parallel}  &\leq \tilde x^\mt {\btheta}^*_\parallel + \frac{C_T}{T - o(T)} + \frac{\alpha_T}{\sqrt{T - o(T)}}
\end{align*}

Substitute them  back and we have that with probability at least $1-2\delta$,
\begin{align}\label{eq:necessity_2delta}
\quad & \tilde x^\mt    \btheta^*_{\parallel}    >  x_a^\mt    \btheta^*_{\parallel}  + x_a^\mt\hat\btheta_{T,\perp} + \text{CB}_T(x_a)  -  \text{CB}_T(\tilde x)  - \frac{2C_T}{T - o(T)} -\frac{2\alpha_T}{\sqrt{T - o(T)}} , \forall x_a \not = \tilde{x}
\end{align}


Let us first consider the case of $x_a \nparallel \tilde x$. Substitute Claim~\ref{claim:cb} and  Claim~\ref{claim:cb_lowerbound} back and we have with probability at least $1-4\delta$
\begin{align*}
\tilde x^\mt    \btheta^*_{\parallel}    >&  x_a^\mt    \btheta^*_{\parallel}  + x_a^\mt\hat\btheta_{T,\perp} +   \Theta\left(\sqrt{\log (T/\delta)} \left(\frac{1}{\sqrt{o(T)}} -\frac{1}{\sqrt{T-o(T)}}\right)\right)\\
&-O\left(\sqrt{\frac{\log (T/\delta)}{T-o(T)}}\right) - \frac{2C_T}{T - o(T)} -\frac{2\alpha_T}{\sqrt{T - o(T)}} , \forall x_a \nparallel \tilde{x}
\end{align*}
Taking $T\to\infty$ and noticing that the last three terms diminish to $0$ faster than the third term, there must exists a $T_0$ such that for any $T>T_0$,
\begin{align}\label{eq:necessary_appendix_ineq}
\quad & \tilde x^\mt    \btheta^*_{\parallel}    >  x_a^\mt    \btheta^*_{\parallel}  + x_a^\mt   \hat\btheta_{T,\perp}, \forall   x_a \not = \tilde{x}
 \end{align}
satisfies when $x_a \nparallel \tilde{x}$.

Now we consider the special case that some $x_a \parallel \tilde x$ and show that the above inequality is still strict. Let $x_a = c \tilde x$. If $|c| > 1$, we have $\text{CB}_T(x_a) - \text{CB}_T(\tilde x) = (c-1)\text{CB}_T(\tilde x) > 0$. If $|c| < 1$, since $\tilde x$ is pulled linear times for any large $t$ with sublinear cost, then  $\tilde x^\mt\hat\theta_{t,\parallel} > 0$; otherwise the cost would be linear. We directly have $\tilde x^\mt\hat\theta_{t,\parallel} =  x_a^\mt\hat\theta_{t,\parallel} + (1-c)\tilde x^\mt\hat\theta_{t,\parallel} >  x_a^\mt\hat\theta_{t,\parallel}$. This leads to $\tilde x^\mt\theta^*_\parallel >  x_a^\mt\theta^*_\parallel$ (inequality~\eqref{eq:necessary_appendix_ineq}) since $x_a \perp \hat\btheta_{T, \perp}$.

Combining the two cases, we know there must exist a $\hat\btheta_{T,\perp}$ that satisfies inequality~\eqref{eq:necessary_appendix_ineq} (the first constraint of CQP \eqref{lp:attackability}), $\tilde x \perp \hat\btheta_{T,\perp}$ (the second constraint  of CQP \eqref{lp:attackability}), and makes the objective of CQP \eqref{lp:attackability} larger than $0$.
To satisfy the last constraint, we consider LinUCB with the choice of L2-regularization parameter $\lambda$ that guarantees $\Vert\hat\btheta_t\Vert_2 < 1$ given the data for large enough $T$ and any $t<T$. Note that ridge regression is equivalent to minimizing square loss under some constraint $\Vert\hat\btheta_t\Vert_2 \leq c$, and there always exists a one-to-one correspondence between $\lambda$ and $c$ (one simple way to show the correspondence is using KKT conditions). Therefore, we are guaranteed to find a $\lambda$ that gives us $c = 1 - \zeta$ where $\zeta$ is an arbitrarily small constant. Then we know that $\hat\btheta_{T,\perp}$ satisfies $\Vert \hat\btheta_T = \hat\btheta_{T, \parallel}  +\hat\btheta_{T,\perp}\Vert_2 \leq c < 1$. We prove the last constraint $\Vert \tilde\btheta = \btheta^*_{\parallel}  +\hat\btheta_{T,\perp}\Vert_2 \leq 1$ by the fact that $\Vert\hat\btheta_{T, \parallel}  +\hat\btheta_{T,\perp}\Vert_2 < 1$ and $\Vert\btheta^*_{\parallel} -  \hat\btheta_{T, \parallel} \Vert_2$ is arbitrarily small for large enough $T$ according to Lemma~\ref{claim:theta_parallel}.

Now we proved that there exists a $\hat\btheta_{T, \perp}$ that satisfies all the constraints in CQP~\eqref{lp:attackability} with positive objective, which means the optimal objective $\epsilon^*$ must also be positive. This however contradicts our assumption $\epsilon^*\leq 0$, implying that such LinUCB is not attackable by any attack strategy. 
\end{proof}


\section{Details on Effective Attacks Without Knowledge of Model Parameters}
We now prove the theorems of using Two-stage Null Space Attack (Algorithm~\ref{alg:attack}) against LinUCB and Robust Phase Elimination.

\subsection{Proof of Theorem~\ref{theorem:attackLinUCB}}\label{sec:prooftheorem1}
\begin{proof}[Proof of Theorem~\ref{theorem:attackLinUCB}]

We first prove that for a large enough $T$, Algorithm \ref{alg:attack}  will correctly assert the attackability with probability at least $1-\delta$. We rely on the following lemma to show  $\tilde\btheta_\parallel$ estimated in step 11 of Algorithm \ref{alg:attack} is close to the true parameter $\btheta^*_\parallel$.

\begin{lemma}[Estimation error of $\tilde\btheta_\parallel$]\label{lemma:theta_parallel}
Algorithm~\ref{alg:attack} estimates $\btheta^*_\parallel$ by
\begin{equation}
    \tilde\btheta_\parallel =   \frac{\sum_{i=1}^{n(\tilde x)} r_i(\tilde x)}{n(\tilde x)\Vert\tilde x\Vert_2^2}\tilde x. 
\end{equation}
With probability at least $1 - \delta$, the estimation error is bounded by 
\begin{equation}
    \Vert  \tilde\btheta_\parallel -  \btheta^*_\parallel \Vert_2 \leq \sqrt{\frac{2R^2\log(1/\delta)}{n}}
\end{equation}
where the rewards have  $R$-sub-Gaussian noise.
\end{lemma}
\begin{proof}

$\btheta^*_\parallel$ is the projected vector of $\btheta^*$ onto $\tilde x$, which is
\[
\btheta^*_\parallel =   \frac{\tilde x^\mt \btheta^*}{\Vert\tilde x\Vert_2^2}\tilde x 
\] as defined in Eq~\eqref{eq:theta_parallel}. Though we need to estimate the vector $\tilde\btheta_\parallel\in \mathbb{R}^d$, we actually only need to estimate the scale  of it by $\hat l = \frac{\sum_{i=1}^{n(\tilde x)} r_i(\tilde x)}{n(\tilde x)\Vert\tilde x\Vert_2^2}$, since the direction is known to be $\tilde x$. 
Based on Hoeffding’s inequality, the estimation error of averaged rewards on $\tilde x$ is bounded by \begin{equation}
    P\left(\left\lvert \frac{\sum_{i=1}^{n(\tilde x)} r_i(\tilde x)}{n(\tilde x)} - r^*(\tilde x) \right\rvert \geq \sqrt{\frac{2R^2\log(1/\delta)}{n(\tilde x)}}\right) \leq \delta
\end{equation}
where $r^*(\tilde x) = \tilde x^\mt \btheta^*$.
Considering $\Vert\tilde x\Vert_2^2 = 1$ and we finish the proof.
\end{proof} 
In the first stage, the bandit algorithm will pull the target arm $\tilde x$ for $T_1 -O(\sqrt{T_1})$ times, since $\tilde x$ is the best arm according to $\btheta_0$. According to Lemma~\ref{lemma:theta_parallel}, with probability at least $1-\delta$ the estimation error  is bounded as
\[\Vert\tilde\btheta_\parallel -  \btheta^*_\parallel \Vert_2 \leq \sqrt{\frac{2R^2\log(1/\delta)}{T_1 -O(\sqrt{T_1})}}.\]
As a result, with a large enough $T_1$, the error of $\tilde x$'s reward estimation satisfies
\[
\lvert\tilde x^\mt\tilde\btheta_\parallel - \tilde x^\mt\btheta^*_\parallel\rvert \leq \Vert\tilde x \Vert_2 \Vert\tilde\btheta_\parallel -  \btheta^*_\parallel \Vert_2 \leq \sqrt{\frac{2R^2\log(1/\delta)}{T_1 -O(\sqrt{T_1})}} \leq \epsilon^*.
\] 
Thus solving CQP \eqref{lp:attackability} with $\tilde\btheta_\parallel$ replacing $\btheta^*_\parallel$ and we could correctly assert attackability with an estimated positive index $\tilde\epsilon^*$ when the environment is attackable with index $\epsilon^*$. 

\begin{remark}
 From the analysis above, we notice that the adversary requires sufficiently large $T_1$ to collect enough rewards on the target arm, in order to make the correct attackability assertion. When $T_1$ is not large enough, the attackability test may have false positive or false negative conclusions. We empirically test the error rate and report the results in Additional Experiments section. 

\end{remark} 

We now prove the success and total cost of the proposed attack.  The analysis relies on the ``robustnes'' property of LinUCB stated in Lemma~\ref{lemma:robustrness_ridgereg}, which is restated and proved below.

\begin{lemma}[Robustness of ridge regression]
Consider LinUCB with ridge regression under poisoning attack. Let $S'_t = \sum_{\tau\in\{1 \dots t\}, x_{a_\tau}\neq\tilde x} |\Delta_\tau|$ be the total corruption on non-target arms until time $t$ and assume every corruption on target arm is bounded by $\gamma$. Then for any $t= 1 \dots T$, with probability at least $1-\delta$ we have 
 \begin{equation}
    \Vert \tilde\btheta- \hat{\btheta}_t\Vert _{A_t} \leq \alpha_t + S'_t/\sqrt{\lambda} + \gamma\sqrt{t}
 \end{equation} 
 where $\alpha_t = \sqrt{d\log\left(\frac{1+t/\lambda}{\delta}\right)} + \sqrt{\lambda}$.  
\end{lemma}
\begin{proof}
Based on the closed form solution of ridge regression, we have
\begin{eqnarray*}
\hat{\btheta}_t =  \tilde\btheta - \lambda \bA_t^{-1}  \tilde\btheta    + \bA_t^{-1} \sum_{\tau=1}^t x_{a_\tau}[\eta_\tau + \Delta_\tau]
\end{eqnarray*}
Therefore, using ideas from \cite{Improved_Algorithm}, we can have 
\begin{eqnarray*}
    \Vert  \hat{\btheta}_t -   \tilde\btheta\Vert _{\bA_t} &\leq &  \lambda^{1/2}\Vert \btheta^*\Vert_2    + \Vert\sum_{\tau=1}^t x_{a_\tau}\eta_\tau\Vert_{\bA_t^{-1}} + \Vert \sum_{\tau=1}^t x_{a_\tau}\Delta_\tau\Vert _{\bA_t^{-1}}\\
    &\leq &  \alpha_t + \Vert \sum_{\tau=1}^t x_{a_\tau}\Delta_\tau\Vert _{\bA_t^{-1}}\\
    &\leq &  \alpha_t + \Vert \sum_{\tau\in\{1 \dots t\}, x_{a_\tau}\neq\tilde x} x_{a_\tau}\Delta_\tau\Vert _{\bA_t^{-1}} + \Vert \sum_{\tau\in\{1 \dots t\}, x_{a_\tau} = \tilde x} x_{a_\tau}\Delta_\tau\Vert _{\bA_t^{-1}}\\
    &\leq & \alpha_t + \Vert \sum_{\tau\in\{1 \dots t\}, x_{a_\tau}\neq\tilde x} x_{a_\tau}\Delta_\tau\Vert _{\bA_t^{-1}} + \Vert \gamma n(\tilde x)\tilde x \Vert_{\bA_t^{-1}}\\
    &\leq & \alpha_t + \Vert \sum_{\tau\in\{1 \dots t\}, x_{a_\tau}\neq\tilde x} x_{a_\tau}\Delta_\tau\Vert_2/\sqrt{\lambda} + \Vert \gamma n(\tilde x)\tilde x \Vert_{\bA_t^{-1}}\\
    &\leq & \alpha_t + S'_t/\sqrt{\lambda} + \Vert \gamma n(\tilde x)\tilde x \Vert_{\bA_t^{-1}}\\
    &\leq & \alpha_t + S'_t/\sqrt{\lambda} + \frac{\gamma n(\tilde x)}{\sqrt{n(\tilde x)}}\\
    &\leq & \alpha_t + S'_t/\sqrt{\lambda} + \gamma\sqrt{t}\\
\end{eqnarray*} 
with probability at least $1-\delta$, where $n(\tilde x)$ is the times target arm has been pulled. The second step is based on the definition of $\alpha_t$ and introduces the high probability bound, the fourth step is because we have $|\Delta_\tau| < \gamma$ if $x_{a_\tau} = \tilde x$; the fifth step is because of $ \bA_t \succeq \lambda\bI $; and the second last inequality follows Eq~\eqref{eq:x_Ainv}. Finally, notice that $n(\tilde x) \leq t$ and we finish the proof.
\end{proof}

Let us first analyze the attack in the first stage. Denote $R_{T}(\btheta)$ as the regret of LinUCB until round $T$, where $\btheta$ is the ground-truth parameter. We know from~\cite{Improved_Algorithm} that if the rewards are all generated by $\btheta$ then with probability at least $1-\delta$ we have
\begin{equation}
    R_{T}(\btheta) =  \alpha_T\sqrt{dT\log\left(\frac{1+T/\lambda}{\delta}\right)} = O(d\sqrt{T}\log(T/\delta))
\end{equation}
where $\alpha_t = \sqrt{d\log\left(\frac{1+t/\lambda}{\delta}\right)} + \sqrt{\lambda}$. 
Then the attack in the first $T_1$ rounds based on $\btheta_0$ should make the bandit algorithm pull $\tilde{x}$ at least $T_1 - R_{T_1}(\btheta_0)/\epsilon_0^*$ times. According to Lemma~\ref{lemma:theta_parallel}, with probability at least $1-\delta$ parameter estimation error is bounded by  \begin{equation}\label{eq:theta_parallel_error}
\Vert  \tilde\btheta_{\parallel, T_1} -  \btheta^*_\parallel \Vert_2 \leq \sqrt{2\log(1/\delta)}/\sqrt{T_1 - R_{T_1}(\btheta_0)/\epsilon_0^*} \leq 2\sqrt{2\log(1/\delta)}/\sqrt{T_1}
\end{equation} for large enough $T_1$. This means we have 
\begin{equation}\label{eq:gamma}
    \gamma = \Vert \tilde x^\mt\big( \tilde\btheta_{\parallel, T_1} -  \btheta^*_\parallel\big) \Vert \leq 2\sqrt{2\log(1/\delta)}/\sqrt{T_1}
\end{equation}

Now we prove the attack is successful with high probability. Consider the regret of the LinUCB against $\tilde\btheta$ as the ground-truth parameter. The estimation error in $ \hat\btheta_t - \tilde\btheta$ has three sources: the sub-Gaussian noise, the rewards on non-target arms in the first stage generated by $\btheta_0$ (the rewards on the target arm are corrected to $\tilde\btheta$ in step 19-20 in Algorithm \ref{alg:attack}), and the unattacked rewards on target arm in the second stage generated by $\btheta^*$. According to Lemma~\ref{lemma:robustrness_ridgereg}, with probability at least $1-3\delta$, we have 
\[ \Vert\hat\btheta_t - \tilde\btheta\Vert_{\bA_t} \leq \alpha_t + R_{T_1}(\btheta_0)/\sqrt{\lambda} + 2\sqrt{2t\log(1/\delta)}/\sqrt{T_1}, t > T_1.\]

To show the number of rounds pulling non-target arms, we first look at the regret against $\tilde\btheta$, i.e., $R_{T}(\tilde\btheta)$.
\begin{align*}
R_{T}(\tilde\btheta) 
&\leq \sum_{t=1}^T \left(\tilde x^\mt\tilde\btheta-x_{a_t}^\mt\tilde\btheta\right)\\
&\leq \sum_{t=1}^T \left(\tilde x^\mt\hat\btheta_t + \text{CB}_t(\tilde x)-x_{a_t}^\mt\tilde\btheta\right)\\
&\leq \sum_{t=1}^T \left(x_{a_t}^\mt\hat\btheta_t + \text{CB}_t( x_{a_t})-x_{a_t}^\mt\tilde\btheta\right)\\
&\leq \sum_{t=1}^T 2\text{CB}_t( x_{a_t})\\
&\leq 2\sqrt{T\sum_{t=1}^T \text{CB}^2_t( x_{a_t})}\\
&\leq 2 \Vert\hat\btheta_T - \tilde\btheta\Vert_{\bA_T} \sqrt{T\sum_{t=1}^T  \lVert x \rVert^2_{\bA_t^{-1}}}\\
&\leq 2 \big(\alpha_T + R_{T_1}(\btheta_0)/\sqrt{\lambda} + 2\sqrt{2T\log(1/\delta)}/\sqrt{T_1}\big) \sqrt{dT\log\left(\frac{1+T/\lambda}{\delta}\right)}
\end{align*}
holds with probability at least $1-3\delta$. And LinUCB will pull non-target arms at most $R_{T}(\tilde\btheta)/\epsilon^*$ times, which can be bounded by
\begin{align*}
 R_{T}(\tilde\btheta) /\epsilon^* &\leq 2 \left(\alpha_T + R_{T_1}(\btheta_0)/\sqrt{\lambda} + 2\sqrt{2T\log(1/\delta)}/\sqrt{T_1}\right) \sqrt{dT\log\left(\frac{1+T/\lambda}{\delta}\right)} /\epsilon^*\\
 &\leq 2 \left(\sqrt{dT\log\left(\frac{1+T/\lambda}{\delta}\right)} + \sqrt{\lambda} + R_{T_1}(\btheta_0)/\sqrt{\lambda} + 2\sqrt{2T\log(1/\delta)}/\sqrt{T_1}\right) \sqrt{dT\log\left(\frac{1+T/\lambda}{\delta}\right)} /\epsilon^*\\
\end{align*}
and is in the order of
\begin{equation}\label{eq:non_target_pulls}
    O\left(d\left(\sqrt{\log(T/\delta)} + \sqrt{T_1}\log{(T_1/\delta)}+\sqrt{T\log(1/\delta)}/\sqrt{T_1}\right)\sqrt{T\log (T/\delta)}/\epsilon^*\right)
\end{equation}
The $\sqrt{T_1}\log{(T_1/\delta)}$ term is due to the ``corrupted'' rewards of non-target arms observed in the first stage. Setting $T_1 = T^{1/2}$ gives us the minimum number of rounds pulling non-target arms in $\tilde O(T^{3/4})$ according to Eq~\eqref{eq:non_target_pulls}.

Now we prove the total cost $C(T)$. Note that in order to make the attack ``stealthy'', we inject sub-Gaussian noise $\tilde \eta_t$ on perturbed reward to make it stochastic. We separate the total cost by the cost on changing the mean reward and the cost of sub-Gaussian noise as
\begin{equation}\label{eq:cost_decompose}
C(T) = \sum_{t=\{1..T\}, \tilde r_t \neq r_t} |\Delta r_t| \leq \sum_{i=1}^N|x_{a_i}^\mt (\tilde\btheta - \btheta^*)| + \sum_{i=1}^N|\tilde \eta_i|
\end{equation}
where $i \in \{t=\{1..T\}: \tilde r_t \neq r_t\}$. Let $N = |\{i\}|$ be the total rounds of attack. Since we know the times attacking non-target arms is bounded by Eq~\ref{eq:non_target_pulls} and attack target arm at most $T_1$ times, we have with probablity $1-\delta$,
\begin{equation}\label{eq:n_rounds_attack}
N =T_1 + O\left(d\left(\sqrt{\log(T/\delta)} + \sqrt{T_1}\log{(T_1/\delta)}+\sqrt{T\log(1/\delta)}/\sqrt{T_1}\right)\sqrt{T\log (T/\delta)}/\epsilon^*\right)   = \tilde O(T^{3/4})    
\end{equation} when setting $T_1 = T^{1/2}$.

Notice that since $\tilde \eta_t$ is $R$-sub-Gaussian, its absolute value $|\tilde \eta_t|$ is also $R$-sub-Gaussian, and $\mathbb{E}[|\tilde \eta_t|] < L$ for some constant $L$ following Proposition 2.5.2 in \cite{vershynin2018high}. According to the general Hoeffding's inequality, with probability at least $1-\delta$,
\begin{equation}
\sum_{i=1}^N |\tilde \eta_i| \leq NL + \sqrt{NL\log(2/\delta)} = O(N + \sqrt{N\log(1/\delta)})
\end{equation}
 Thus the second term of  Eq~\eqref{eq:cost_decompose}
has the order of $O(N) = \tilde O(T^{3/4})$. The first term of  Eq~\eqref{eq:cost_decompose} is bounded by $2N + 2T_1$ because each attack changes the mean reward at most 2 except the compensation step in line 20, and the reward compensation on target arm can be bounded by $2T_1$ because target arm is pulled at most $T_1$ times in the first stage. Overall, with probability at least $1-4\delta$
\begin{equation}\label{eq:cost_final}
    C(T)\leq \sum_{i=1}^N|x_{a_i}^\mt (\tilde\btheta - \btheta^*)| + \sum_{i=1}^N|\tilde \eta_i| \leq 2N + 2T_1 + NL + \sqrt{NL\log(2/\delta)} = \tilde O(N) = \tilde O(T^{3/4})
\end{equation}
when setting $T_1 = T^{1/2}$.

\end{proof}

\subsection{Proof Sketch of Corollary~\ref{corollary:attackrobustPhE}}
The proof is similar to the proof of Theorem~\ref{theorem:attackLinUCB}, and thus we only explain the difference here.

Instead of using Lemma~\ref{lemma:robustrness_ridgereg} to analyze the impact of perturbed rewards generated by $\btheta_0$ (against $\tilde\btheta$) collected in the first stage, we know RobustPhE has an additional regret term $O(S^2)$ for corruption $S$ (assuming $S$ is unknown to the bandit algorithm).  
Since the bandit algorithm observes $T_1$ rewards in the first stage, $S \leq 2T_1$ and the additional regret due to rewards from first stage is $\tilde O(T_1^2)$. For the unattacked rewards on target arm in the second stage generated by $\btheta^*$, we view them as rewards generated by $\tilde\btheta$ with misspecification error $\gamma$, i.e., $|\tilde x^\mt(\btheta^* - \tilde\btheta)|\leq \gamma $. Proposition 5.1 in \cite{lattimore2020learning} showed that the phase elimination algorithm with  misspecification $\gamma$ has additional regret in $O(\gamma\sqrt{d} T\log(T))$. With $\gamma \leq 2\sqrt{2\log(1/\delta)}/\sqrt{T_1}$ by  Eq~\eqref{eq:gamma}, the total regret is \[R_{T}(\tilde\btheta) = O\Big(d\sqrt{T}\log (T/\delta) + \sqrt{d}T\log(T)\log (1/\delta)/\sqrt{T_1} + T_1^2\Big)\] with probability at least $1-2\delta$. Therefore, we have with probability at least $1-2\delta$, the attack strategy  will fool RobustPhE to pull  non-target arms at most $O\Big((d\sqrt{T}\log (T/\delta) + \sqrt{d}T\log(T)\log (1/\delta)/\sqrt{T_1} + T_1^2)/\epsilon^*\Big)$ rounds. 

Similar to Eq~\eqref{eq:n_rounds_attack}, we bound the total rounds of attacking RobustPhE by
\begin{equation}\label{eq:n_rounds_attack_robustphe}
N =T_1 + O\Big(\big(d\sqrt{T}\log (T/\delta) + \sqrt{d}T\log(T)\log (1/\delta)/\sqrt{T_1} + T_1^2\big)/\epsilon^*\Big)  
\end{equation}
From Eq~\eqref{eq:cost_final}, we know the total cost is in the same order as the rounds of attack.
So with probability at least $1-3\delta$ the total cost is
\[
O\Big(T_1 + (d\sqrt{T}\log (T/\delta) + \sqrt{d}T\log(T)\log (1/\delta)/\sqrt{T_1} + T_1^2)/\epsilon^*\Big).
\] 
Setting $T_1 = T^{2/5}$ gives us the minimum attack cost $\tilde O(T^{4/5})$, and the non-target arms are pulled at most $\tilde O(T^{4/5})$ rounds. 

\begin{remark}

Note that we bound the total corruption by $T_1$, which means the adversary does not need to compensate the rewards on the target arm as shown in line 20 in Algorithm \ref{alg:attack}. The robustness of RobustPhE allows us to carry over the rewards in the first stage while LinUCB does not. 

\end{remark}

\subsection{Attack under unknown $T$}\label{sec:multi-stage}

Our two-stage null space attack algorithm requires that $T$ is known for the convenience of analysis. Here we discuss a promising idea that leveraging the \emph{doubling trick} to extend the attack to the case of unknown $T$, which will lead to a multi-stage attack that repeatedly adjusts the target parameter $\tilde\theta$ at each stage. More concretely, to attack LinUCB without knowing $T$, the adversary can start with a pre-specified  horizon $T_0$ and run the two-stage attack. Once the actual horizon reaches $T_0$, the adversary will expand the horizon to $T_0^2$ (i.e., the ``doubling'' trick), and views   previous $T_0$ rounds as the new first stage, re-calculates target parameter $\tilde\theta$ using rewards from the $T_0$ rounds and runs the new attack strategy until $T_0^2$. Using this exponential horizon sequence $\{T_i = T_0^{2^i}, i\in\mathbb{N}\}$, we have a multi-stage attack on LinUCB with unknown time horizon. Similarly, to attack RobustPhE, we will need to adjust the horizon sequence to be $\{T_i = T_0^{(5/2)^i}, i\in\mathbb{N}\}$, which will lead to a similar multi-stage attack on RobustPhE. We believe this direction is feasible based on our current analysis, and more thorough and complete proof should be the target of our next work.

\section{Additional Experiments}\label{sec:exp_supp}
\begin{figure}[ht]
\vspace{-1mm}
\centering
\setlength\tabcolsep{.5pt}
\includegraphics[width=7.5cm]{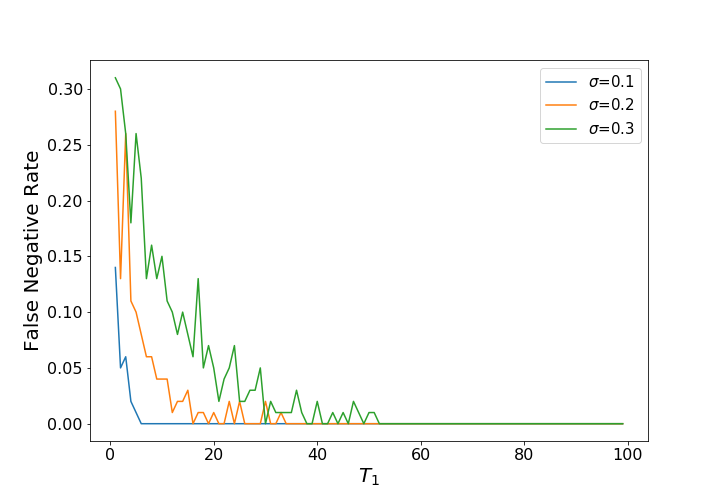}
\vspace{-1mm}
\caption{False negative rate of attackability test}\label{fig:appendix}
\vspace{-1mm}
\end{figure}


In Figure \ref{fig:appendix}, we study the false negative rate of the attackability test in Algorithm~\ref{alg:attack}, i.e., how often the adversary mistakenly asserts that an attackable environment is not attackable. As we explained in Proof of Theorem 1,  
the wrong assertion is because of using estimated $\tilde\btheta_\parallel$ instead of the ground-truth bandit parameter. In this experiment, we consider a challenging attackable three-arm environment with $\mathcal{A}=\{x_1 = (0, 1), x_2 = (0.11, 1.1), x_3 = (-2, 0)\}$, $\tilde x = x_1$ and $\btheta^* = (0, 0.5)$. By solving CQP~\eqref{lp:attackability}, we have attackability index  $\epsilon^* = 0.005$ and certificate $\tilde\btheta_\perp = (-0.5, 0)$\footnote{We introduce arm $x_3$ to guarantee the first dimension of $\tilde\btheta_\perp$ cannot be smaller than $-0.5$. Comparing $\tilde x$ and $x_2$ and we can see the optimal solution is $\epsilon^* = 0.005$.}. We test two-stage null space attack against LinUCB with $T = 10,000$ and the adversary will test the attackability after the first $T_1 = T^{1/2} = 100$ rounds. We vary $T_1$ from $1$ to $100$ to see how many iterations is sufficient for attackability test. We report averaged results of 100 runs. We also vary the standard derivation $\sigma$ of Gaussian noise from  0.1 to 0.3. In Figure~\ref{fig:appendix}, we can see that the false negative rate is almost zero when $T_1>50$, suggesting $T_1 = 100$ is sufficient.  When $\sigma = 0.1$ the adversary only needs around 10 rounds to make a correct assertion. We also notice the false negative rate becomes higher under a larger noise scale. As suggested in Lemma~\ref{lemma:theta_parallel}, the error in  $\tilde\btheta_\parallel$ estimation is larger if noise scale is larger or the number of target arm's rewards $n(\tilde x)$ is smaller, which highly depends on $T_1$. Larger error means CQP~\eqref{lp:attackability} with $\tilde\btheta_\parallel$ is more likely to be unfeasible and gives false negative assertion. However, $T_1 = 100$ is still enough for the attackability test when $\epsilon^* = 0.005$.